\documentclass{article}

\PassOptionsToPackage{numbers, compress}{natbib}
\usepackage[preprint]{neurips_2026}

\makeatletter
\renewcommand{\@noticestring}{Preprint. Under review.}
\makeatother

\usepackage{graphicx}
\usepackage[utf8]{inputenc}
\usepackage[T1]{fontenc}
\usepackage{hyperref}
\usepackage{url}
\usepackage{booktabs}
\usepackage{amsfonts}
\usepackage{nicefrac}
\usepackage{microtype}
\usepackage{xcolor}

\usepackage{amsmath}
\usepackage{amssymb}
\usepackage{mathtools}
\usepackage{amsthm}

\usepackage{amsmath,amsfonts,bm}

\def\eqref#1{equation~\ref{#1}}
\def\1{\bm{1}}

\DeclareMathAlphabet{\mathsfit}{\encodingdefault}{\sfdefault}{m}{sl}
\SetMathAlphabet{\mathsfit}{bold}{\encodingdefault}{\sfdefault}{bx}{n}

\DeclareMathOperator*{\argmin}{arg\,min}

\usepackage{multicol}
\usepackage{wrapfig,lipsum}
\usepackage{framed}
\usepackage{enumitem}
\usepackage{kotex}

\usepackage{multirow, array}
\usepackage{algorithm}
\usepackage{algpseudocode}
\usepackage{cleveref}
\Crefname{appsec}{Appendix}{Appendices}
\crefname{appsec}{Appendix}{Appendices}
\Crefname{appsubsec}{Appendix}{Appendices}
\crefname{appsubsec}{Appendix}{Appendices}
\Crefname{appsubsubsec}{Appendix}{Appendices}
\crefname{appsubsubsec}{Appendix}{Appendices}
\crefname{equation}{eq.}{eqs.}
\Crefname{equation}{Eq.}{Eqs.}

\definecolor{cadmiumgreen}{rgb}{0.0, 0.42, 0.24}
\definecolor{cornellred}{rgb}{0.7, 0.11, 0.11}
\definecolor{Gray}{gray}{0.9}

\hypersetup{citecolor=MidnightBlue, linkcolor=BrickRed}
\hypersetup{
  linkcolor = cornellred,
  citecolor  = cadmiumgreen,
  colorlinks = true,
}

\title{\textsc{COVA}riance-Induced Fairness Gap Penalty for Subgroup-Fair Clustering}

\author{
  \textbf{Kyungseon Lee}$^{1,}$\thanks{Equal contribution.}
  \quad
  \textbf{Hankyo Jeong}$^{1,}$\footnotemark[1]
  \quad
  \textbf{Kunwoong Kim}$^{2,}$\footnotemark[1]\\[-0.1em]
  \textbf{Kwanho Lee}$^{1}$
  \quad
  \textbf{Yongdai Kim}$^{1,}$\thanks{Corresponding author.}\\[0.45em]
  $^{1}$Department of Statistics, Seoul National University 
  $^{2}$KAIST AI\\[0.35em]
  \texttt{ppleeqq@snu.ac.kr \quad
          hankyo2273@snu.ac.kr \quad
          kwkim.online@gmail.com}\\[-0.05em]
  \texttt{khlee0527@snu.ac.kr \quad
          ydkim903@snu.ac.kr}
}

\theoremstyle{plain}
\newtheorem{theorem}{Theorem}[section]

\theoremstyle{definition}
\newtheorem{definition}{Definition}[section]
 
\theoremstyle{remark}
\newtheorem{remark}[theorem]{Remark}

\begin{document}

\maketitle

\begin{abstract}
    Fair clustering aims to make cluster assignments independent of sensitive attributes, but this goal becomes challenging when multiple sensitive attributes jointly define many subgroups.
    In such settings, directly extending existing fair clustering algorithms is computationally expensive or numerically unstable, especially when the number of subgroups grows exponentially and some subgroups contain only a few instances.
    To address these challenges, we define a subgroup-fairness gap for clustering and derive a covariance-based surrogate that exactly matches this gap.
    We then introduce a continuous relaxation of the surrogate, enabling efficient gradient-based optimization and yielding our proposed algorithm, \textsc{Cova-FC}.
    We also show that subgroup fairness alone does not imply marginal fairness, and extend our framework to capture a subgroup-marginal-fairness gap.
    Experiments on benchmark datasets show that \textsc{Cova-FC} achieves competitive cost-fairness trade-offs and improves computational efficiency over existing baselines in both subgroup and higher-order marginal settings.
\end{abstract}

\section{Introduction}
Algorithmic fairness is a critical requirement for decision-making systems that affect access to opportunities, resources, or services, as systematic biases and disparities can harm protected and underrepresented populations.
Accordingly, regulations and policies increasingly require trustworthy and non-discriminatory algorithms, making fair learning important both ethically and practically \citep{goodman2017european, mehrabi2021survey}.
Fairness in machine learning has been studied across various tasks, including classification \citep{pmlr-v80-agarwal18a, donini2018empirical, goh2016satisfying, KIM2022441, kim2025fairness}, regression \citep{ADW19, NEURIPS2020_51cdbd26}, generation \citep{friedrich2023fair, xu2018fairgan}, and clustering \citep{bera2019fair, chierichetti2017fair, kleindessner2019guarantees}.

Among these, \textit{fair clustering} \citep{chierichetti2017fair} aims to construct clusters whose memberships are not strongly associated with sensitive attributes.
% Concretely, it seeks solutions in which no cluster disproportionately concentrates a sensitive attribute combination, and the cluster assignment is not predictive of sensitive attributes.
The technical goal is to achieve a desirable trade-off between clustering cost and fairness, since improving fairness typically comes with an increase in cost \citep{10.1287/opre.1100.0865, 9541160}.
% Various learning algorithms for fair clustering have been developed by \cite{backurs2019scalable, chierichetti2017fair, esmaeili2021fcbc, kim2025fair, zeng2023deep, ziko2021variational}.

However, most existing fair clustering algorithms were developed mainly for a single sensitive attribute \citep{backurs2019scalable, chierichetti2017fair, esmaeili2021fcbc, kim2025fair, ziko2021variational}.
In many modern applications, however, individuals are associated with multiple sensitive attributes, whose intersections can reveal fairness violations even though a given algorithm is fair for each sensitive attribute.
For example, a lending model may equalize approval rates across gender and race overall, yet the subgroup ``female \& minority race'' may still receive markedly lower approval rates.
This motivates the need for subgroup fairness, where
a \emph{subgroup} is defined by a specific combination of sensitive attribute values.
Although there is a large number of studies for subgroup-fair classification \citep{blum_et_al:LIPIcs.ITCS.2020.55,Hu_2024,kearns2018empirical, kim2025draf, Lai2025fairicp,molina2023boundingapproximatingintersectionalfairness}, not much work has been done for subgroup-fair clustering, which is the main theme of this paper.

Existing fair clustering algorithms could be extended to this setting by treating multiple sensitive attributes as a single sensitive attribute with multiple categories,
but this naive extension creates substantial computational and optimization challenges.
Given \(q\) binary sensitive attributes, the number of subgroups becomes \(2^q\), and many subgroups can be very sparse.
Algorithms using linear programming (LP), such as \citep{bera2019fair, esmaeili2021fcbc, kim2025fair}, become computationally expensive as the number of subgroup-dependent constraints increases.
Penalty-based methods such as VFC \citep{ziko2021variational} and FairKM \citep{abraham2020fairkm} would be promising but they have their own limitations.
VFC becomes numerically unstable when very small subgroups exist because their fairness penalties depend on subgroup sizes. 
% FairKM relies on hard-assignment updates that limit flexibility in gradient-based optimization \citep{abraham2020fairkm}.
% FairKM only considers hard cluster assignments and thus subgroup fairness could not be achieved in particularly when
% there are tiny subgroups, but its modification for soft cluster assignments would not be easy. 
FairKM relies on hard cluster assignment, making it difficult to finely match the desired subgroup proportions across clusters when some subgroups contain only a few instances, 
but its modification for soft cluster assignment would not be easy. 
% See Appendix \ref{app:vfc_deep_compare} for details.
Thus, subgroup-fair clustering remains largely underexplored, especially in settings where the number of subgroups is large.

In this paper, we develop a learning algorithm for subgroup-fair clustering with multiple sensitive attributes.
Formally, assume that we are given $q$ binary sensitive attributes, and let the $i$-th instance have a sensitive attribute vector $s_i \in \{0,1\}^q,$ which induces $2^{q}$ subgroups.

%%%%%%%%%%%%%%%%%%%%%%%%%%%%%%%%%%%%
%%%%%%%%%%%%%%%%%%%%%%%%%%%%%%%%%%%%
%%%%%%%%%%%%%%%%%%%%%%%%%%%%%%%%%%%%
\paragraph{Our approach.}
% Inspired by \cite{kearns2018preventing}, we first define a subgroup-fairness gap
% and propose a smooth surrogate of the subgroup-fairness gap 
% called the covariance-based fairness gap,
% and develop a learning algorithm which minimizes the clustering cost under the fairness constraint by use of a gradient-descent optimization. 
% The resulting assignment updates are parallelizable across instances, enabling scalable optimization.
% Experiments on benchmark datasets show that our proposed algorithm 
% \textsc{Cova-FC} (\textit{\textbf{COVA}riance-based \textbf{F}air \textbf{C}lustering}) achieves superior cost-fairness trade-offs together with computational efficiency and numerical stability.
Inspired by \cite{kearns2018preventing}, we first define a subgroup-fairness gap for clustering and derive a smooth covariance-based surrogate for this gap.
The surrogate leads to a gradient-based algorithm that optimizes the clustering cost with a fairness penalty and supports parallel per-instance assignment updates, making the method scalable to large data with many subgroups.
We call the proposed algorithm \textsc{Cova-FC} (\textit{\textbf{COVA}riance-based \textbf{F}air \textbf{C}lustering}).

\paragraph{Additional marginal control.}
% In practice, we would want not only subgroup fairness but also marginal fairness: fairness for each sensitive attribute.
% Unfortunately, subgroup fairness does not imply marginal fairness and so we need an additional tool to ensure marginal fairness.
% We propose a modification of \textsc{Cova-FC} that can ensure subgroup and marginal fairness simultaneously without much additional computation.
Beyond subgroup fairness, marginal fairness (fairness for each sensitive attribute) may also be required in practice, since subgroup fairness alone may fail to control marginal fairness.
To handle such cases, we further extend \textsc{Cova-FC} to jointly control subgroup and marginal fairness with little additional computation.

The main contributions of this work are:
\begin{itemize}[noitemsep=1pt, topsep=1pt]
    \item[$\diamond$]
     We develop a subgroup-fair clustering algorithm under multiple sensitive attributes
     that is numerically stable and easily scalable to achieve competitive cost-fairness trade-offs.

    \item[$\diamond$]
    We propose a way to ensure subgroup and marginal fairness simultaneously without much additional computation burden.
    
    \item[$\diamond$]
    We conduct extensive numerical studies to illustrate the advantages of \textsc{Cova-FC} over existing fair clustering algorithms.
\end{itemize}

%%%%%%%%%%%%%%%%%%%%%%%%%%%%%%%%%%%%
\section{Subgroup-fairness gap and its relaxation}\label{sec:partial}

Suppose we are given a set of instances $\{(x_i, s_i)\}_{i=1}^n$, where each $x_i \in \mathbb{R}^d$ is a $d$-dimensional feature vector and $s_i=(s_{i}^{(1)}, \ldots, s_{i}^{(q)})^\top\in\{0,1\}^q$ is a vector of $q$ binary sensitive attributes.
For each $s\in\{0,1\}^q$, let $n_s := \sum_{i=1}^n \mathbb{I}(s_i=s).$
For $N\in\mathbb{N}$, we write $[N]=\{1,\ldots,N\}$.
We consider only binary sensitive attributes for notational simplicity but our algorithm can be extended to multi-valued sensitive attributes.

\paragraph{Clustering objective.}

Following prior works for fair clustering \citep{bera2019fair, esmaeili2021fcbc,  kim2025fair, li2020deep, zeng2023deep, ziko2021variational}, we consider the $K$-means clustering problem.
For a given number of clusters $K \in \mathbb{N},$ let $\mathbf{\mu} = ( \mu_{1}, \ldots, \mu_{K} )^{\top}$ with $\mu_{k} \in \mathbb{R}^{d}, k \in [K]$ denote the cluster centroids, 
and let $\mathbf{A} = (A_{1}, \ldots, A_{n})$ with $A_{i} \in [K], i \in [n]$ denote the (hard) cluster assignment vector.
Later, we overload $\mathbf A$ to also denote a soft cluster assignment $A_i \in \mathbb S^K$ where $\mathbb{S}^K$ is the simplex on $\mathbb{R}^K$ (e.g., \Cref{sec:cont_relax}), whenever the meaning is clear from context.
The standard $K$-means clustering problem is to find the cluster centroids $\mu$ and cluster assignments $\mathbf{A}$
that minimizes the \textit{clustering cost}, i.e., the average of within-cluster distances:
$ \min_{\mathbf{\mu}, \mathbf{A}} \frac{1}{n} \sum_{i=1}^{n} L_{i}(\mathbf{\mu}, \mathbf{A}), $
where $L_{i}(\mathbf{\mu}, \mathbf{A}) := \sum_{k=1}^{K} \mathbb{I}(A_{i} = k) \Vert x_{i} - \mu_{k} \Vert^{2}$ is the cost for the $i$-th instance.
% Note that this objective is equivalent to solving $\min_{\mathbf{\mu}} \frac{1}{n} \sum_{i=1}^{n} \min_{k \in [K]} \Vert x_{i} - \mu_{k} \Vert^{2}.$

%%%%%%%%%%%%%%%%%%%%%%%%%%%%%%%%%%
%%%%%%%%%%%%%%%%%%%%%%%%%%%%%%%%%%
\subsection{Subgroup-fairness gap}\label{sec:gap-def}

Our primary goal is to find the cluster centroids and cluster assignment $(\mathbf{\mu},\mathbf{A})$ that achieve (i) not only a low clustering cost $\frac{1}{n}\sum_{i} L_{i} (\mathbf{\mu}, \mathbf{A}),$ (ii) but also fairness with respect to multiple sensitive attributes. In this subsection, we propose a measure to quantify the level of subgroup unfairness.

For each $k \in [K],$ define the overall assignment probability as
$
p_{k}(\mathbf A) := \frac{1}{n} \sum_{i=1}^{n} \mathbb{I}(A_{i} = k).
$
For each subgroup $s \in \{0,1\}^q$ with $n_s > 0,$ define the conditional assignment probability in subgroup $s$ as
$
p_{k\mid s}(\mathbf A) := \frac{1}{n_{s}}\sum_{i:s_i=s}\mathbb I(A_i=k).
$
Let $\pi_s := n_s/n$ denote the subgroup proportion.
Inspired by the idea of subgroup fairness in \cite{kearns2018preventing}, we define the following subgroup-fairness gap for clustering.
    
\begin{definition}[Subgroup-fairness gap]\label{def:Delta_subgroup}
The \textit{subgroup-fairness gap} of $\mathbf A$ is defined as
    \begin{equation}\label{eq:Delta_subgroup}
        \Delta(\mathbf A)
        :=
        \max_{s \in \{0,1\}^q,\, n_s>0}
        \pi_s
        \sum_{k=1}^K
        \left|p_k(\mathbf A)-p_{k\mid s}(\mathbf A)\right|.
    \end{equation}
\end{definition}
We say $\mathbf{A}$ is \textit{subgroup-fair at level $\varepsilon$} if $\Delta(\mathbf A) \le \varepsilon.$
The weight $\pi_s$ reflects the size of subgroup $s,$ and prevents tiny subgroups from dominating the gap $\Delta.$
Without this weight, we cannot reduce the subgroup-fairness gap under a certain level that depends on the sizes of tiny subgroups.
See \Cref{sec:weight_necessity} for an example.

%%%%%%%%%%%%%%%%%%%%%%%%%%%%%%%%%%
%%%%%%%%%%%%%%%%%%%%%%%%%%%%%%%%%%
\paragraph{Objective of learning subgroup-fair clusters.}

Based on the subgroup-fairness gap defined in \Cref{def:Delta_subgroup}, a natural goal is to minimize the clustering cost under the subgroup fairness constraint:
\begin{equation}\label{eq:subgroup-fair-problem}
    \min_{\mathbf{\mu}, \mathbf A}\ \frac{1}{n}\sum_{i=1}^n L_i(\mathbf \mu, \mathbf A)
    \quad
    \text{subject to}
    \quad
    \Delta(\mathbf A) \le \varepsilon.
\end{equation}
A common way to handle such a constrained optimization is to solve a penalized problem, i.e., minimize $\frac{1}{n}\sum_i L_i(\mathbf \mu, \mathbf A) + \lambda \Delta(\mathbf A)$ where $\lambda>0$ is the Lagrangian multiplier.

%%%%%%%%%%%%%%%%%%%%%%

%%%%%%%%%%%%%%%%%%%%%%
\subsection{$\textup{CR}(\mathbf A)$: a covariance-based surrogate for \(\Delta(\mathbf A)\)}\label{sec:proposed_pen_cr}

\paragraph{Motivation: challenges in using \(\Delta\) as the fairness constraint.}

Directly using \(\Delta(\mathbf A)\) as the fairness penalty poses a critical computational challenge.
This is because, unlike the standard \(K\)-means objective, the penalized objective is not separable in \(\mathbf A\).
In \(K\)-means, the clustering cost \(\frac{1}{n} \sum_{i=1}^n L_i(\mathbf\mu,\mathbf A)\) decomposes over \(i\in[n]\), so each assignment \(A_i\) can be updated in parallel given \(\mathbf\mu\).
In contrast, \(\Delta(\mathbf A)\) is not separable because it involves the absolute difference between the overall assignment proportion \(p_k(\mathbf A)=\frac{1}{n}\sum_{i=1}^n \mathbb I(A_i=k)\) and the subgroup-conditional assignment proportion \(p_{k\mid s}(\mathbf A)=\frac{1}{n_s}\sum_{i:s_i=s}\mathbb I(A_i=k)\), both of which depend on sums of assignments over certain instances.
As a result, changing a single assignment \(A_i\) affects these proportions as well as $\Delta(\mathbf{A})$, making parallel updates of $A_i$ difficult.
To address the problem, we introduce in this subsection a surrogate that rewrites \(\Delta(\mathbf A)\) in a more optimization-friendly form.

\paragraph{Proposed surrogate fairness gap.}
% To resolve the issue, we introduce a \textit{covariance-based} fairness gap as a surrogate for \(\Delta(\mathbf A)\), which is a key contribution of this work.
We now define the \textit{covariance-based} surrogate.
This idea comes from the insight that subgroup unfairness can be understood as how well subgroup memberships predict the cluster assignments, and so subgroup unfairness can be viewed as statistical dependence between the cluster assignment and the subgroup membership.

For each subgroup membership \(s \in \{0,1\}^q\), define $c_{is} := 2\mathbb I(s_i=s)-1$ and $\bar c_s := \frac{1}{n}\sum_{i=1}^n c_{is} = 2\pi_s - 1.$
Let
$ \textup{C}(\mathbf A,s,\beta) := \frac{1}{2n}\sum_{i=1}^n \left(\sum_{k=1}^K \beta_k \mathbb I(A_i=k)\right) (c_{is}-\bar c_s), $
and define the \textit{\textbf{C}ova\textbf{R}iance-based (CR) subgroup-fairness gap} as
\begin{equation}\label{eq:CR_def}
    \textup{CR}(\mathbf A)
    :=
    \max_{s\in\{0,1\}^q:\,n_s>0}
    \max_{\beta\in\mathcal B_\infty}
    \textup{C}(\mathbf A,s,\beta),
    \quad
    \textup{where}
    \quad
    \mathcal B_\infty := \{\beta\in\mathbb R^K : \|\beta\|_\infty \le 1\}.
\end{equation}
% where \(\mathcal B_\infty := \{\beta\in\mathbb R^K : \|\beta\|_\infty \le 1\}\).

For a given subgroup membership \(s \in \{0, 1 \}^{q} \), the term \((c_{is}-\bar c_s)\) is the centered subgroup sign, and \(\sum_{k=1}^K \beta_k \mathbb I(A_i=k)\) is a signed score assigned to the cluster label \(A_i\) with \(\beta_k\in[-1,1]\).
Hence, \(\textup{C}(\mathbf A,s,\beta)\) is the empirical covariance between the signed cluster assignment and the centered membership for subgroup \(s\).
Maximizing over \(\beta\in\mathcal B_\infty\) chooses the sign pattern that aligns best with the imbalance in subgroups, and maximizing over \(s\) selects the most violated subgroup.
Therefore, \(\textup{CR}(\mathbf A)\) can also be understood as a tool to quantify the largest linear dependence between the cluster assignment \(\mathbf A\) and the subgroup membership.
The following \Cref{thm:Delta_CR_equal} provides a theoretical support for this interpretation, whose proof is deferred to \Cref{sec:appen-proofs}.
% The proof is deferred to \Cref{sec:appen-proofs}.

\begin{theorem}[Equivalence between \(\Delta\) and \(\textup{CR}\)]\label{thm:Delta_CR_equal}
    For any hard cluster assignment \(\mathbf A\), we have $ \Delta(\mathbf A) = \textup{CR}(\mathbf A). $
\end{theorem}

Notably, \(\textup{CR}(\mathbf A)\) is not an approximate surrogate, but an exact surrogate for \(\Delta(\mathbf A)\). 

\subsection{Optimization challenges of \(\textup{CR}(\mathbf A)\) and a continuous relaxation}\label{sec:cont_relax}

Although \(\textup{CR}(\mathbf A)\) is exactly equivalent to \(\Delta(\mathbf A)\), it is still difficult to be minimized directly in practice.
The difficulty comes from two discrete components: the subgroup membership \(s \in \{0,1\}^q\) and the hard cluster assignment \(\mathbf A \in [K]^n\). 
Even a small change in \(\mathbf A\) can switch the most violated subgroup abruptly, which often makes the optimization unstable.
To improve optimization stability, we relax both the hard cluster assignment and the discrete maximization over subgroups into continuous ones.

\paragraph{(1) Relaxation of the cluster assignment \(\mathbf A\).}

We first replace the hard cluster assignment vector with soft cluster assignment vector.
Specifically, we let \(\mathbf A = (A_1,\ldots,A_n)^\top \in [0,1]^{n\times K}\), where each row \(A_i=(A_{i1},\ldots,A_{iK})^\top\) lies in the simplex \(\mathbb S^K\).
Here, \(A_{ik}\) represents the assignment probability of \(x_i\) to cluster \(k\).
Hard assignment is recovered as the special case in which each \(A_i\) is one-hot.
% Thus, any statement for soft $\mathbf A$ can be applied to the hard case directly.

Under this relaxation, we replace \(\mathbb I(A_i=k)\) by \(A_{ik}\) in the covariance term, and rewrite the clustering cost as \(L_i(\mathbf{\mu}, \mathbf{A}) = \sum_{k=1}^{K} A_{ik}\Vert x_i-\mu_k\Vert^2\).
This makes both the clustering loss and the covariance term differentiable with respect to \(\mathbf A\), which enables gradient-based updates.

\paragraph{(2) Relaxation of the worst-case subgroup.}
We next relax the hard maximization over subgroups.
Let \(\mathcal S_+ := \{s\in\{0,1\}^q : n_s>0\}\) be the set of observed subgroups and let \(M_+ := |\mathcal S_+|\).
Since empty subgroups do not contribute to the empirical fairness gap, we optimize only over \(\mathcal S_+\).
Instead of selecting a single observed subgroup, we allow a convex combination of subgroup indicators through a weight vector \(\mathbf v \in \mathbb S^{M_+}\).
This yields the continuous surrogate
\begin{equation}\label{eq:J_simplex_def_refined}
\overline{\textup{CR}}(\mathbf A)
:=
\max_{\mathbf v \in \mathbb S^{M_+}} \max_{\beta\in\mathcal B_\infty}
\frac{1}{2n}\sum_{i=1}^{n}
\left(\sum_{k=1}^K \beta_k A_{ik}\right)
 \mathbf v ^\top (c_{i}-\bar c),
\end{equation}
where $c_i := (c_{is})_{s\in \mathcal S_+}$ and $\bar c := \frac{1}{n}\sum_{i=1}^n c_i.$
Optimizing over \(\mathbf v\) avoids abrupt switching of the worst-case subgroup and leads to a smoother and continuous objective.
The following theorem shows that $\overline{\textup{CR}}(\mathbf A)$ is a continuous relaxation of $\Delta(\mathbf{A}),$ whose proof is deferred to Appendix \ref{sec:appen-proofs}. 

\begin{theorem}[Relationship between \(\Delta\) and \(\overline{\textup{CR}}\)]\label{thm:Delta_CRbar_equal}
    For any hard cluster assignment \(\mathbf A\), we have
    $ \Delta(\mathbf A)=\overline{\textup{CR}}(\mathbf A). $
\end{theorem}

\section{Considering marginal fairness}
\label{sec:marg_sub}

The subgroup-fairness gap introduced in the previous section provides a tractable objective for handling multiple sensitive attributes.
However, it evaluates fairness only at the finest subgroup level and may therefore provide loose control over marginal fairness.
This can be socially unacceptable in applications, even when subgroup fairness is achieved.
Indeed, a small subgroup-fairness gap does not necessarily imply a small marginal-fairness gap.
See \Cref{sec:subgroup_marginal_counterexample} for a counterexample and \Cref{sec:large-number-subgroup} for empirical evidence.
To address this limitation, we modify the subgroup-fairness gap to enforce subgroup and marginal fairness simultaneously.

\subsection{Modification of the subgroup-fairness gap for marginal fairness}\label{sec:subgroup-not-marginal} 

For \(j\in[q]\) and \(a\in\{0,1\}\), define the marginal subgroup as
\(
    W_{j,a}:=\{s\in\{0,1\}^q:s_j=a\}.
\)
This set contains all of the subgroups whose \(j\)-th sensitive attribute takes the value \(a\).
We call $W_{j,0}$ and $W_{j,1}$ the \emph{marginal subgroups} corresponding to the sensitive attribute $s_j.$
Let
\(
    n_{W_{j,a}}:=\sum_{i=1}^n \mathbb I(s_i\in W_{j,a}),\
    \pi_{W_{j,a}}:=\frac{n_{W_{j,a}}}{n},
\)
and define
\(
    p_{k\mid W_{j,a}}(\mathbf A)
    :=
    \frac{1}{n_{W_{j,a}}}
    \sum_{i:s_i\in W_{j,a}}\mathbb I(A_i=k).
\)
We say that \textit{$\mathbf{A}$ is marginally fair} with respect to $s_j$ if $\Delta_j(\mathbf{A})$ is small where
$$\Delta_j(\mathbf{A})=\max_{a\in\{0,1\} }
\pi_{W_{j,a}}  \sum_{k=1}^K
    \left|p_k(\mathbf A)-p_{k\mid W_{j,a}}(\mathbf A)\right|.$$
In turn, we say that $\mathbf{A}$ is marginally fair if 
$\Delta_{\textup{marg}}(\mathbf A)=\max_{j\in [q]} \Delta_j(\mathbf{A}),$ 
which we call the marginal-fairness gap, is small. 
Finally, we propose a modification of the subgroup-fairness gap in \Cref{def:Delta_sub_marg} that ensures subgroup fairness and marginal fairness simultaneously.
\begin{definition}[Subgroup-marginal-fairness gap]\label{def:Delta_sub_marg}
    The subgroup-marginal-fairness gap of $\mathbf{A}$ is defined as
    $
    \Delta_{\textup{sub-marg}}(\mathbf A)
    :=
    \max\left\{
    \Delta(\mathbf A),
    \Delta_{\textup{marg}}(\mathbf A)
    \right\}.
    $
\end{definition}

\subsection{Higher-order marginal fairness}
\label{sec:higher_order_marginal}

Beyond marginal fairness, one may also consider $r$-th order marginal fairness whose fairness gap is defined as follows.
Let $\theta$ be a subset of $[q]$ with $|\theta|=r$ and define $s_\theta=(s_j, j\in \theta)$.
For $a\in \{0,1\}^r$, define $W_{s_\theta, a}=\{s: s_\theta=a\}$.
We say that
$\mathbf{A}$ is \textit{$r$-th order marginally fair} with respect to $s_\theta$ if $\Delta_{s_\theta}(\mathbf{A})$ is small, where
$$
\Delta_{s_\theta}(\mathbf{A}) := \max_{a\in\{0,1\}^r }
\pi_{W_{s_\theta,a}}  \sum_{k=1}^K
\left|p_k(\mathbf A)-p_{k\mid W_{s_\theta,a}}(\mathbf A)\right|.
$$
Here $\pi_{W_{s_\theta,a}}$ and $p_{k\mid W_{s_\theta,a}}(\mathbf A)$
are defined similarly to $\pi_{W_{j,a}}$ and $p_{k\mid W_{j,a}}(\mathbf A).$    
In turn, we say that $\mathbf{A}$ is $r$-th order marginally fair if 
$\Delta_{\textup{marg}, r}(\mathbf A)=\max_{\theta \subset [q]: |\theta|=r} \Delta_{s_\theta}(\mathbf{A}),$ 
which we call the $r$-th order marginal-fairness gap, is small. 
Similar to \Cref{def:Delta_sub_marg}, we define subgroup-marginal-fairness gap up to the $r$-th order as the following.
\begin{definition}[Subgroup-marginal-fairness gap up to the $r$-th order]\label{def:Delta_sub_marg_k}
    The subgroup-marginal-fairness gap of $\mathbf{A}$ up to the $r$-th order is defined as
    $
    \Delta_{\textup{sub-marg}, r}(\mathbf A)
    :=
    \max\left\{
    \Delta(\mathbf A), 
    \max_{l\in [r]} \Delta_{\textup{marg}, l}(\mathbf A)
    \right\}.
    $
\end{definition}

\subsection{Surrogate subgroup-marginal-fairness gap}
\label{sec:sub-mar_gap}

In this subsection, we show that the surrogate subgroup-fairness gap introduced in Section \ref{sec:cont_relax}
can be modified easily for the subgroup-marginal-fairness gap up to the $r$-th order.
For notational simplicity, we introduce a notion of subgroup-subsets.
A \textit{subgroup-subset} $W$ is nothing but a subset of $\{0,1\}^q.$
Note that any $r$-th order marginal subgroups as well as subgroups themselves are subgroup-subsets.w

Let $\mathcal{W}$ be the set of all subgroup-subsets involving the subgroup-marginal-fairness gap, which includes
$W_{s_\theta,a}$ for all $\theta\subset [q]$ with $1\le|\theta|\le r$ and $a\in \{0,1\}^{|\theta|}$ as well
as $\{s\}$ for all $s\in \{0,1\}^q.$ Then, we have 
$\Delta_{\textup{sub-marg}, r}(\mathbf A)=\Delta(\mathbf{A}; \mathcal{W}),$ where
\begin{equation}
\label{eq:gap-W}
    \Delta(\mathbf A;\mathcal W)
    :=
    \max_{W\in \mathcal{W}:\, n_{W}>0}
    \pi_W
    \sum_{k=1}^K
    \left|p_k(\mathbf A)-p_{k\mid W}(\mathbf A)\right|.
\end{equation}
By choosing $\mathcal{W}$ accordingly, we can make $\Delta(\mathbf A;\mathcal W)$
equal either the subgroup-fairness gap or the subgroup-marginal-fairness gap (up to the $r$-th order).
We are going to develop a fair clustering algorithm under the fairness constraint on $\Delta(\mathbf A;\mathcal W)$
for a given $\mathcal{W}.$

% Let $\mathcal{W}=\{W_m\}_{m=1}^{M}$ with $M := \bigl|\mathcal{W}\bigr|$ and $M_{+}:=\bigl|\{m \in [M]: n_{W_m}>0\}\bigr|$.
% For each $m \in [M]$, define $c_{im}:=2\mathbb{I}(s_i \in W_m)-1$ and $\bar{c}_m := \frac{1}{n}\sum_{i=1}^{n}c_{im}$, and write $c_i:=(c_{im})_{m \in M_{+}}$ and $\bar{c} := \frac{1}{n}\sum_{i=1}^{n}c_i$.
Let $\mathcal{W}=\{W_1,W_2,...,W_M\}$ with $M := |\mathcal{W}|$,
and let $\mathcal{M}_+ := \{m \in [M]: n_{W_m}>0\}$ with $M_+ := |\mathcal{M}_+|$.
For each $m \in [M]$, define $c_{im}:=2\mathbb{I}(s_i \in W_m)-1$ 
and $\bar{c}_m := \frac{1}{n}\sum_{i=1}^{n}c_{im}$, 
and write $c_i:=(c_{im})_{m \in \mathcal{M}_+}$ 
and $\bar{c} := \frac{1}{n}\sum_{i=1}^{n}c_i$.

Similarly to what we have done in Section \ref{sec:partial}, for hard assignment, define
\begin{equation}\label{eq:CR_W_def}
    \textup{CR}(\mathbf A;\mathcal W)
    :=
    \max_{m \in \mathcal {M}_{+}}
    \max_{\beta\in\mathcal B_\infty}
    \frac{1}{2n}\sum_{i=1}^n
    \left(\sum_{k=1}^K \beta_k \mathbb I(A_i=k)\right)
    (c_{im}-\bar c_m),
\end{equation}
and for soft assignment, replacing \(\mathbb I(A_i=k)\) by \(A_{ik}\) and relaxing the worst-group choice gives
\begin{equation}\label{eq:CRbar_W_def}
    \overline{\textup{CR}}(\mathbf A;\mathcal W)
    :=
    \max_{\mathbf v\in\mathbb S^{M_+}}
    \max_{\beta\in\mathcal B_\infty}
    \frac{1}{2n}\sum_{i=1}^{n}
    \left(\sum_{k=1}^K \beta_k A_{ik}\right)
    \mathbf v^\top(c_i-\bar c),
\end{equation}
% where $c_i := (c_{im})_{m \in [|\mathcal W|]} \in \mathbb R^{|\mathcal W|}$ and $\bar c := \frac{1}{n}\sum_{i=1}^n c_i$.

\begin{theorem}[Equivalence for fairness over $\mathcal W$]\label{thm:Delta_ms_equiv_hard}
For any hard cluster assignment \(\mathbf A\),
\begin{equation}\label{eq:Delta_ms_hard}
\Delta(\mathbf A;\mathcal W)=\textup{CR}(\mathbf A;\mathcal W)=\overline{\textup{CR}}(\mathbf A;\mathcal W).
\end{equation}
% For soft \(\mathbf A\), $\overline{\textup{CR}}(\mathbf A; \mathcal W)$ extends to a smooth surrogate (by replacing $\mathbb I(A_i = k)$ with $A_{ik}$) that recovers the gap at hard solutions.
\end{theorem}

The proofs are deferred to Appendix \ref{sec:appen-proofs}.
\Cref{thm:Delta_ms_equiv_hard} implies that, $\textup{CR}(\mathbf A; \mathcal W)$ and $\overline{\textup{CR}}(\mathbf A; \mathcal W)$ are both equal to the fairness gap over $\mathcal W$ for hard cluster assignment.
When \(\mathcal W\) is the collection of all subgroups, \Cref{thm:Delta_ms_equiv_hard} reduces to the subgroup-only case.

%%%%%%%%%%%%%%%%%%%%%%%%%%%%%%%%%%%%%%
\section{Proposed algorithm: \textsc{Cova-FC}}
\label{sec:algorithm}

Given a collection $\mathcal{W}$ of subgroup-subsets, we learn fair clusters by minimizing
\begin{equation}\label{eq:obj}
    \frac{1}{n}\sum_{i=1}^n L_i(\mathbf \mu, \mathbf A)
    + \lambda\, \overline{\textup{CR}}(\mathbf A; \mathcal W)
\end{equation}
with respect to the centroids \(\mathbf \mu=(\mu_1,\ldots,\mu_K)\) and the soft assignment matrix $\mathbf A=(A_i)_{i=1}^n,$ for $A_i\in\mathbb S^K.$
Here, \(\lambda\ge0\) controls the trade-off between clustering cost and fairness.

We call the resulting algorithm \textit{\textbf{Cova}riance-based \textbf{F}air \textbf{C}lustering} (\textsc{Cova-FC}).
The algorithm alternates between two steps.
First, for fixed \(\mathbf A\), it updates the adversarial variables \((\beta,\mathbf v)\) to identify the most violated subgroup-subset direction in \(\overline{\textup{CR}}\).
Second, for fixed \((\beta,\mathbf v)\), it updates the soft assignment \(\mathbf A\) by gradient descent and updates the centroids \(\mathbf \mu\) as in standard K-means algorithms with soft assignment.
The assignment update has a closed-form gradient and decomposes over instances, which enables efficient parallel updates.
Detailed update formulas and the full pseudocode are provided in \Cref{app:cova_fc_algorithm}.
Below, we explain two computational advantages of \textsc{Cova-FC}.

%%%%%%%%%%%%%%%%%%%
\paragraph{Computational efficiency.}

An additional advantage of using \(\overline{\textup{CR}}\) as the fairness penalty is its computational efficiency, i.e., it offers a lower computational complexity than existing algorithms.
Each iteration of \textsc{Cova-FC} has complexity
\(\mathcal O(nKd+n(K+2^q))\), where \(nKd\) is the cost for the computation of \(K\)-means algorithm and \(n(K+2^q)\) is the additional cost for adversarial updates for fairness.
Table \ref{tab:complexity} compares the computational complexity of \textsc{Cova-FC} with other baselines.
It shows that \textsc{Cova-FC} avoids (i) the cubic dependence in $n$ of LP-based methods and (ii) the large term \(nK2^q\) of VFC.
Further details as well as large-scale experimental results are given in \Cref{app:complexity}.

\begin{table}[h!]
    \centering
    \caption{Comparison of fair clustering methods in terms of computational complexity.}
    \label{tab:complexity}
    \footnotesize
    \setlength{\tabcolsep}{10pt}
    \begin{tabular}{l l}
    \toprule
    Method & Complexity \\
    \midrule
    \textbf{\textsc{Cova-FC}} $\checkmark$ 
    & $n (Kd + K + 2^{q})$ \\
    
    VFC \citep{ziko2021variational}
    & $nK (d + 2^{q})$ \\
    
    FairKM \citep{abraham2020fairkm}
    & $nK(nd + 2^q)$ \\ 
    
    FCA \citep{kim2025fair}
    & $n^2 (n + Kd)$ \\
    
    FCBC \citep{esmaeili2021fcbc}
    & $(nK+2^{q}K)^3$ \\
    
    FRA \citep{bera2019fair}
    & $(nK+2^{q})^3$ 
    \\
    \bottomrule
    \end{tabular}
    \vskip -0.1in
\end{table}

\paragraph{Parallel updates.}
Once the adversarial variables are fixed, the covariance surrogate decomposes over each instance, enabling parallel updates of the cluster assignments, making \textsc{Cova-FC} applicable to large datasets with many subgroups.
In contrast, existing fair clustering algorithms including FCBC \citep{esmaeili2021fcbc}, FCA \citep{kim2025fair}, and FRA \citep{bera2019fair} as well as FairKM \citep{abraham2020fairkm} 
are less suitable for efficient parallel assignment updates because their objectives are not instance-wise decomposable.
% VFC \citep{ziko2021variational} admits parallel updates but has inverse subgroup-cluster mass terms in gradient update term, which make numerical unstability in sparse-subgroup settings (see Appendix \ref{app:vfc_deep_compare}).
VFC \citep{ziko2021variational} admits parallel updates, but its gradient update involves inverse subgroup-cluster masses, which can lead to numerical instability in sparse-subgroup settings (see Appendix~\ref{app:vfc_deep_compare}).
Also \textsc{Cova-FC} runs \(3.2\times\) faster than VFC (see \Cref{app:scalability}) on a large dataset with \(n\) in the millions.
% Runtimes are reported in \Cref{tab:runtime} of \Cref{sec:tradeoff_compare}.

% In contrast, LP-based methods such as FCBC, FCA, and FRA, as well as hard-assignment methods such as FairKM, do not yield the same instance-wise decomposable assignment update.
% Although VFC also admits parallel assignment updates, \textsc{Cova-FC} avoids inverse subgroup-cluster mass terms, which improves numerical stability in sparse-subgroup settings.

Note that our original subgroup-fairness gap is also non-separable in the assignment.
Since both \(p_k(\mathbf A)\) and \(p_{k\mid s}(\mathbf A)\) depend on all assignments, updating a single \(A_i\) changes the gap globally.
In contrast, our covariance surrogate has an instance-wise additive form, enabling parallel updates of the assignment vectors and making \textsc{Cova-FC} more scalable in large subgroup/dataset settings.

%%%%%%%%%%%%%%%%%%%%%%%%%%%%%%%%%%%%
\section{Experiments}\label{sec:exp}

This section presents the results of our experimental analyses.
% Experimental settings such as datasets, baseline methods, and performance metrics are explained in \Cref{sec:exp-setting}.
\Cref{sec:tradeoff_compare} reports cost-fairness trade-offs and runtime on four benchmark datasets (\texttt{Adult}, \texttt{Dutch}, \texttt{Bank}, \texttt{Civilcomments}),
confirming favorable trade-offs and computational efficiency. 
\Cref{sec:large-number-subgroup} validates the subgroup-marginal extension on \texttt{Communities} ($q=18$), where the number of subgroups is large and many subgroups are sparse.
\Cref{sec:comp_comp} studies scalability to large $n$ on \texttt{ACSIncome} ($n \approx 1.66$M) and to high-dimensional features on an image dataset \texttt{CelebA}.
Finally, \Cref{sec:cr-comparison} analyzes the practical effects of the two continuous relaxations introduced in \Cref{sec:cont_relax}.

%%%%%%%%%%%%%%
\subsection{Settings}\label{sec:exp-setting}

\paragraph{Datasets.}

We use seven benchmark datasets, including five tabular datasets (\texttt{Adult}, \texttt{Dutch}, \texttt{Bank}, \texttt{Communities}, and \texttt{ACSIncome}), one text dataset (\texttt{Civilcomments}), and one image dataset (\texttt{CelebA}).
% For tabular datasets, we use numerical (continuous) features only, following prior works \citep{backurs2019scalable,esmaeili2021fcbc,kim2025fair,ziko2021variational}.
% \texttt{Adult} has $16$ subgroups from gender, race, age, and marital status, \texttt{Dutch} has $4$ subgroups from gender and age, and \texttt{Bank} has $3$ subgroups from marital status.
% \texttt{Civilcomments} uses DistilBERT embeddings \citep{Victor2019distillbert} of text comments and has $24$ subgroups.
% \texttt{Communities} \citep{redmond2002data} has $1{,}180$ non-empty subgroups from $18$ binary sensitive attributes, with many small subgroups.
% \texttt{ACSIncome} \citep{10.5555/3540261.3540757} is large-scale ($n=1{,}664{,}500$) and has $16$ subgroups.
% \texttt{CelebA} \citep{liu2015faceattributes} uses $512$-dimensional ResNet-18 image features and has $16$ subgroups.
% All features are standardized to zero mean and unit variance.
See \Cref{tab:datasets} in \Cref{sec:appen-impl} for detailed descriptions of these datasets.

\paragraph{Baselines.}

We compare \textsc{Cova-FC} with well-known existing methods: FCBC \citep{esmaeili2021fcbc}, VFC \citep{ziko2021variational}, FCA \citep{kim2025fair}, FRA \citep{bera2019fair}, and FairKM \citep{abraham2020fairkm}.
Note that FCA faces a bottleneck: it has a step which estimates the joint distributions over subgroups by solving LP, which incurs high computational cost as the number of subgroups increases.
% On large datasets  \texttt{Adult} and \texttt{Civilcomments}, 
% FCA did not complete within a reasonable amount of times, and so we omit FCA from these comparisons and report it only on \texttt{Dutch} and \texttt{Bank}.
On the \texttt{Adult} and \texttt{Civilcomments} datasets, an experiment did not complete within 6 hours, so we omit FCA from the comparisons on these datasets and report it only on \texttt{Dutch} and \texttt{Bank}.

%%%%%%

\paragraph{Performance metrics.}

We use the hard assignments $\hat A_i = \arg\max_k A_{ik}$ for performance evaluation (abbreviated $A_i$ in this section).
Clustering performance is measured by the \(K\)-means clustering cost \(\frac{1}{n}\sum_{i=1}^n L_i(\mathbf \mu,\mathbf A)\).
For fairness, we assess the subgroup fairness gap $\Delta(\mathbf A; \mathcal W)$ defined in \eqref{eq:gap-W}.
We denote the subgroup fairness gap as $\textup{SP} := \Delta(\mathbf A)$, and the 
$l$-th order marginal fairness gap as $\textup{MP}^{(l)} := \Delta_{\textup{marg},l}(\mathbf A)$ defined in \Cref{sec:higher_order_marginal}.
Note that \Cref{sec:partial,sec:marg_sub} define $\Delta$ for binary sensitive attributes ($a \in \{0,1\}$),
 but they can be extended easily for a multi-valued sensitive attribute with $m$ categories by replacing $a \in \{0,1\}$ with $a \in \{0, 1, \ldots, m-1\}.$ 
\medskip

\begin{remark}[Why we use the fairness gap rather than Balance]
    A popularly used fairness measure is the Balance metric \citep{chierichetti2017fair}, however, we do not use it as a primary metric for two reasons:
    \emph{(i)}  Balance is a ratio, so it collapses to zero whenever some cluster receives no instances from a subgroup. It is unavoidable when a subgroup has fewer than $K$ instances (e.g., on \texttt{Civilcomments}).
    \emph{(ii)}  Balance depends on absolute subgroup sizes, so the value is unstable when subgroups have very different sizes.
    The gap remains well-defined under zero cluster counts and reflects group size through the subgroup mass \(\pi_W\), rather than through raw count ratios.
    However, for datasets where Balance is well-defined (i.e., each subgroup is of sufficiently large size), we additionally report subgroup Balance results in \Cref{app:balance-sum-max-gap}.
\end{remark}

\paragraph{Implementation details.}
For the convergence criterion, we terminate the algorithm when the objective in \Cref{eq:obj} changes by less than $2 \times 10^{-5}$ for 15 iterations, with a maximum of $10{,}000$ iterations.
In the main subgroup-fairness experiments, we construct \(\mathcal W\) using all full subgroups only (i.e., $\mathcal W = \{\{s\}: s \in \{0,1\}^q\}$), and fix $K = 10$ for all experiments following prior work \citep{kim2025fair, ziko2021variational}.
Additional details are provided in \Cref{sec:appen-impl}.

%%%%%%%%%%%%%%%%%%%%%
\subsection{Performance comparison}\label{sec:tradeoff_compare}

\paragraph{Cost-fairness trade-off.}

% By examining the cost-fairness trade-off, we quantify the level of both marginal and subgroup fairness that can be achieved for a given clustering cost.
By examining the cost-fairness trade-off, we quantify the level of subgroup fairness achieved for a given clustering cost.
\Cref{fig:gap_trade_offs} shows the Pareto-front trade-offs between clustering cost and subgroup-fairness gap (SP).
\textsc{Cova-FC} attains lower SP than the baselines in most cases. %The corresponding $\mathrm{MP}^{(1)}$ trade-off is reported in \Cref{sec:large-number-subgroup}.
FCBC \citep{esmaeili2021fcbc} and VFC \citep{ziko2021variational} often fail to attain high fairness levels due to numerical overflow, whereas \textsc{Cova-FC} works over a wide range of fairness levels.
As previously explained, FCA \citep{kim2025fair} becomes computationally prohibitive on datasets with many subgroups due to its LP-based updates, so we report its results only on \texttt{Dutch} and \texttt{Bank} (see Appendix \ref{sec:appen-impl}).

\begin{figure*}[h!]
    \vskip -0.1in
    \centering
    \includegraphics[width=0.235\linewidth]{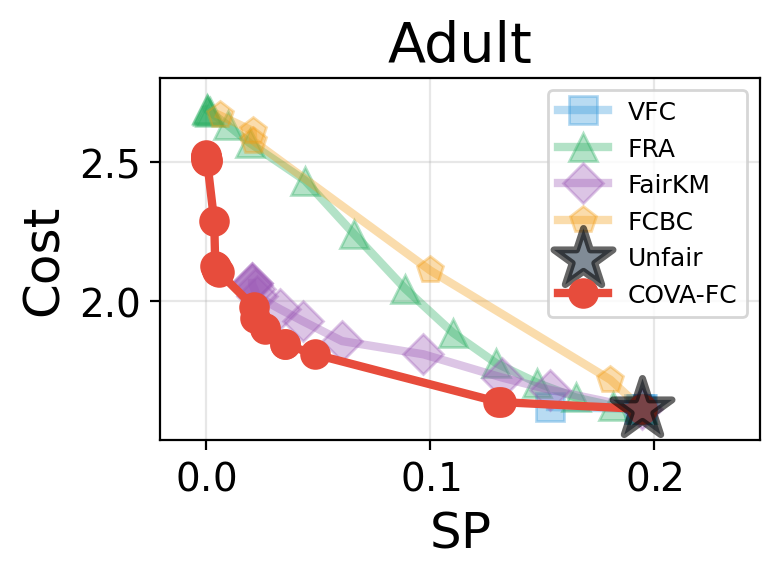}
    \includegraphics[width=0.235\linewidth]{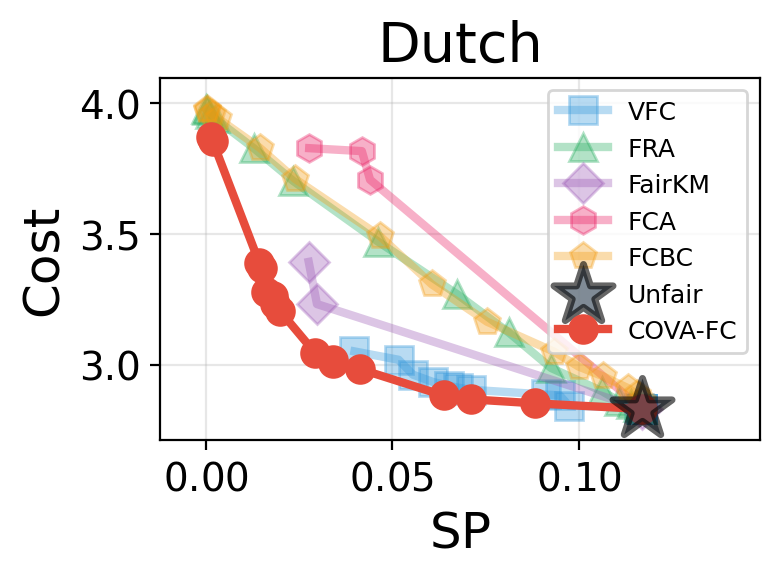}
    \includegraphics[width=0.235\linewidth]{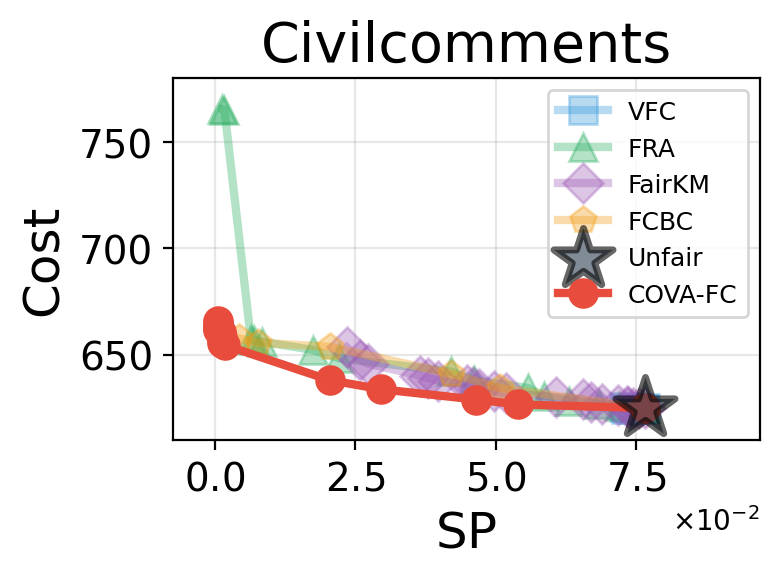}
    \includegraphics[width=0.235\linewidth]{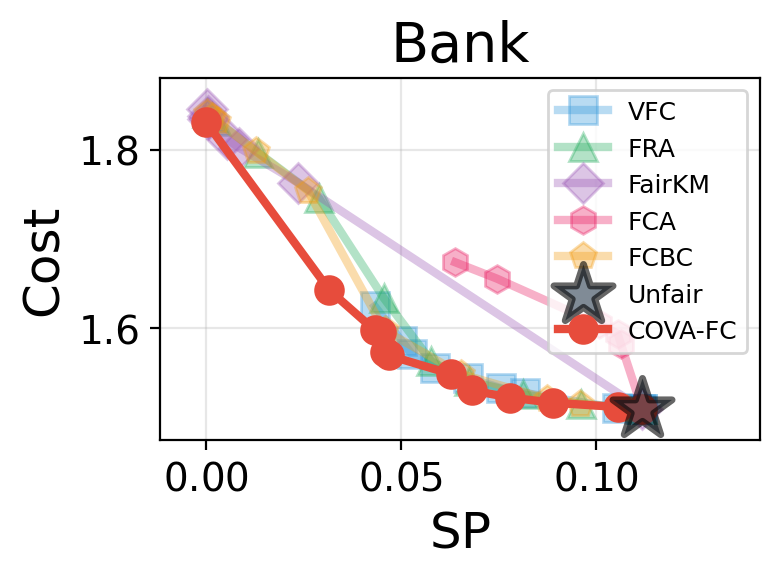}

    \caption{
    Comparison of trade-offs between subgroup-fairness gap (SP) and cost on (left to right) \texttt{Adult}, \texttt{Dutch}, \texttt{Civilcomments}, and \texttt{Bank} datasets.
    }
    \label{fig:gap_trade_offs}
    \vskip -0.1in
\end{figure*}

\newpage
\paragraph{Parallelization and comparison with VFC.}

\begin{wrapfigure}{r}{0.4\textwidth}
    \centering
    \includegraphics[width=0.8\linewidth]{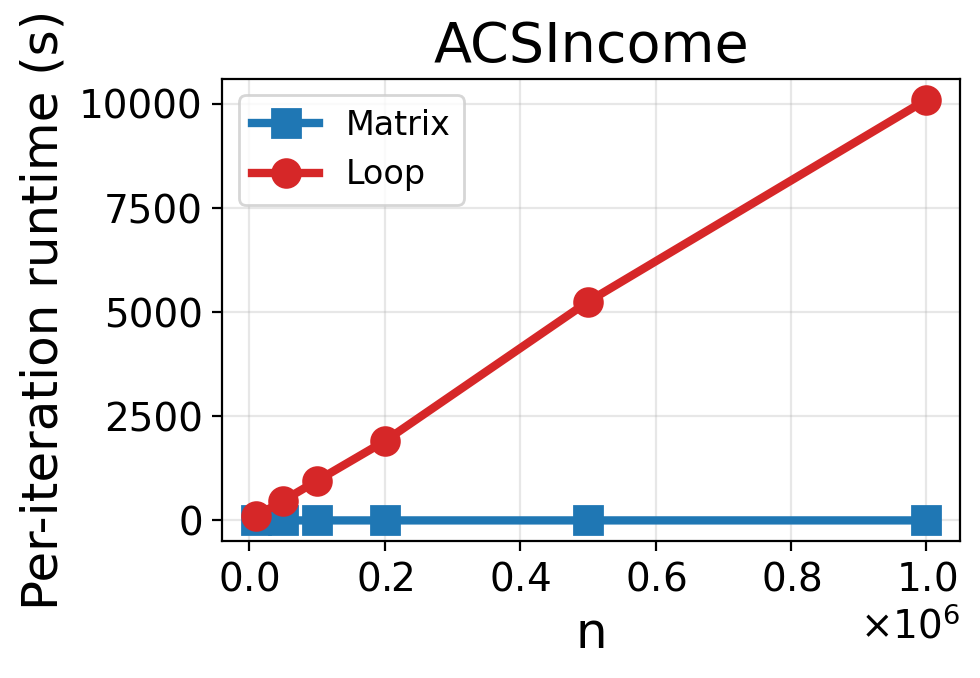}
    \caption{Per-iteration runtime of the loop and matrix implementations of the assignment update on \texttt{ACSIncome}.
    }
    \label{fig:loop_vs_matrix}
\end{wrapfigure}

We compare a per-instance loop implementation of the assignment update against our matrix implementation (and so parallel update) on \texttt{ACSIncome} with varying \(n \in \{10\text{k}, 100\text{k}, 500\text{k}, 1\text{M}\}\).
While the runtime of this loop version grows almost linearly with \(n\), the matrix version grows much more slowly over the same range, confirming that the per-instance assignment updates are effectively parallelized on GPU (\Cref{fig:loop_vs_matrix}).

In contrast, most baselines do not admit the same instance-wise decomposable assignment update: FRA, FCBC, and FCA rely on LP-based optimization, and FairKM uses hard reassignment steps depending on all cluster centroids and sizes of subgroups.
Among the baselines, VFC is the only parallelizable baseline.
With $L_{2}$-normalized data (as used in \cite{ziko2021variational}), \textsc{Cova-FC} still achieves lower $\mathrm{SP}$ at similar cost (\Cref{tab:vfc_cova_l2_sp}).
Further, in \Cref{app:vfc_deep_compare} we discuss numerical instability of VFC in sparse subgroup settings.
That is, VFC involves the inverse of subgroup-cluster masses in the fairness penalty which becomes too large when the sizes of some subgroups are too small.
\textsc{Cova-FC} does not have such a problem.

\begin{table}[h!]
    \centering
    \caption{Comparison between \textsc{Cova-FC} and VFC in terms of the lowest achievable SP.}
    \label{tab:vfc_cova_l2_sp}
    \footnotesize
    \setlength{\tabcolsep}{7pt}
    \begin{tabular}{lcc}
    \toprule
    \multicolumn{3}{c}{SP (Cost)} \\
    \midrule
    Dataset 
    & VFC \citep{ziko2021variational}
    & \textsc{Cova-FC} \\
    \midrule
    \texttt{Adult}   
    & \(0.0158\) \((0.47)\)
    & \(\mathbf{0.0002}\) \((0.50)\) \\
    \texttt{Dutch}   
    & \(0.0146\) \((0.40)\)
    & \(\mathbf{0.0026}\) \((0.41)\) \\
    \texttt{Civilcomments}   
    & \(0.0063\) \((0.86)\)
    & \(\mathbf{0.0009}\) \((0.90)\) \\
    \texttt{Bank}   
    & \(0.0354\) \((0.24)\)
    & \(\mathbf{0.0000}\) \((0.27)\) \\
    \texttt{Communities}  
    & \(0.0171\) \((0.58)\)
    & \(\mathbf{0.0056}\) \((0.86)\) \\
    \bottomrule
    \end{tabular}
    \vskip -0.1in
\end{table}

%%%%%%%%%%%%%%%%%%%%%%%%%%%%%%%%%%%%%%%%%%%%%%%%%%%%%%%%%%
\subsection{Marginal fairness extension}\label{sec:large-number-subgroup}

We evaluate the subgroup-marginal-fairness gap (\Cref{sec:marg_sub}) on the \texttt{Communities} dataset \citep{redmond2002data} with $q=18$ sensitive attributes (1{,}180 subgroups), where many subgroups contain only a few instances.
We compare two variants of \textsc{Cova-FC}: the subgroup-only variant, which uses $\Delta(\mathbf A)$ as the fairness gap, and the subgroup-marginal variant, which uses $\Delta_{\textup{sub-marg},2}(\mathbf A)$ from \Cref{def:Delta_sub_marg_k}.
% That is, the subgroup-marginal variant takes $\mathcal W$ to contain $W_{s_\theta, a}$ for all $\theta \subset [q]$ with $|\theta| \le 2$ and $a \in \{0,1\}^{|\theta|}$, together with all singletons $\{s\}$ for $s \in \{0,1\}^q$, as in \Cref{sec:sub-mar_gap}.
The latter includes first- and second-order marginal groups together with full subgroups.
The subgroup-marginal variant achieves substantially better $\mathrm{MP}^{(1)}$ and $\mathrm{MP}^{(2)}$ than the subgroup-only variant, at the cost of a mild degradation in $\mathrm{SP}$.
\Cref{fig:communities_large_subgroup} reports results averaged over five random splits.
% $\{42, 43, 44, 45, 46\}$.

\begin{figure*}[h!]
    \centering
    \includegraphics[width=0.245\linewidth]{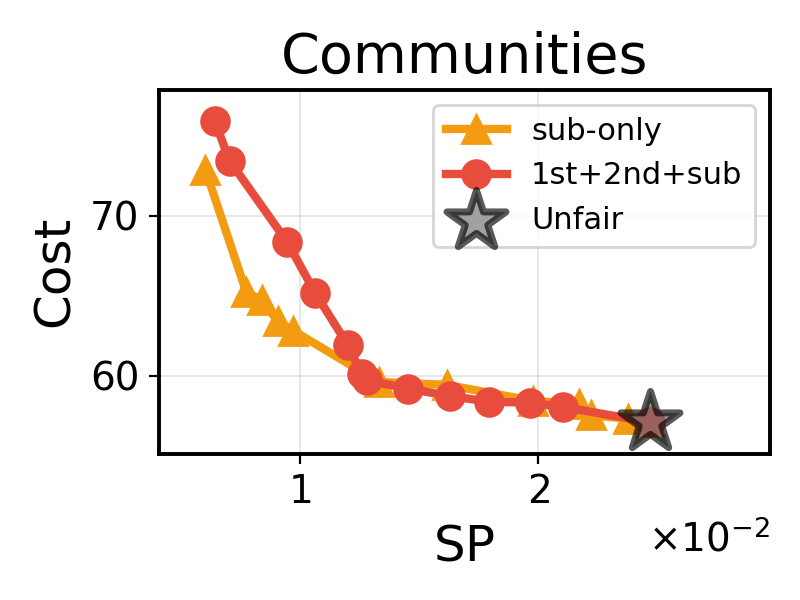}
    \includegraphics[width=0.245\linewidth]{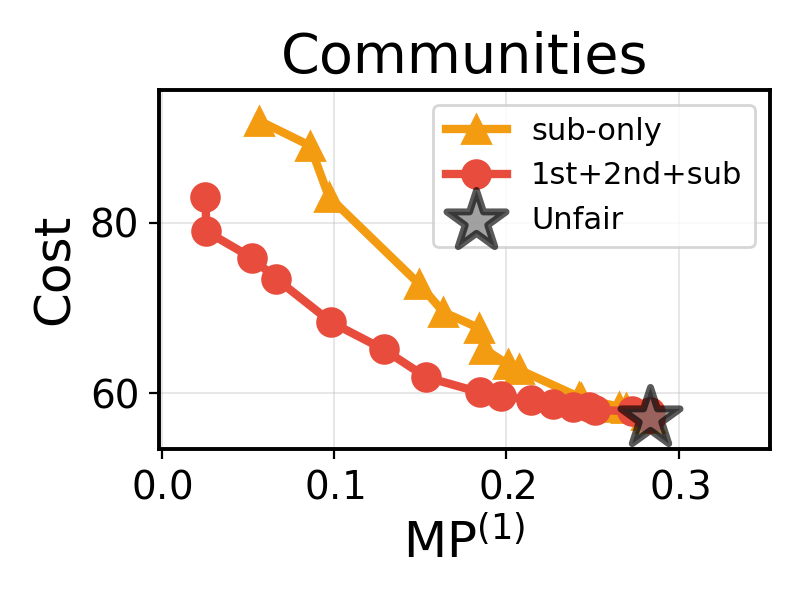}
    \includegraphics[width=0.245\linewidth]{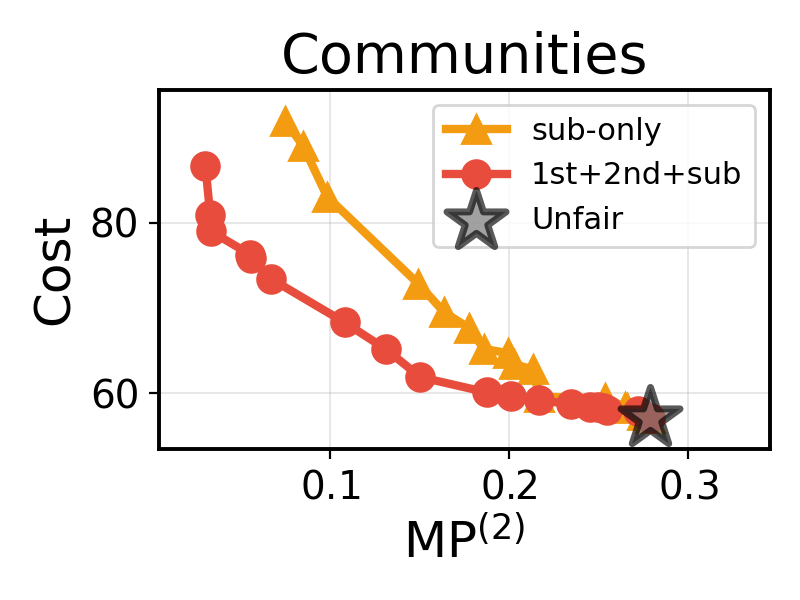}
    \caption{Comparison on \texttt{Communities} between subgroup-only \textsc{Cova-FC} and subgroup-marginal \textsc{Cova-FC} (1st+2nd-order marginal + subgroup):
    (left) $\mathrm{SP}$, (center) $\mathrm{MP}^{(1)}$, (right) $\mathrm{MP}^{(2)}$.}
    \label{fig:communities_large_subgroup}
    \vskip -0.1in
\end{figure*}

We also report the $\mathrm{MP}^{(1)}$ trade-off on the four main benchmark datasets in \Cref{fig:mp_trade_offs}, where the subgroup-marginal variant similarly improves marginal fairness over the subgroup-only variant.

\begin{figure*}[h!]
    \centering
    \includegraphics[width=0.245\linewidth]{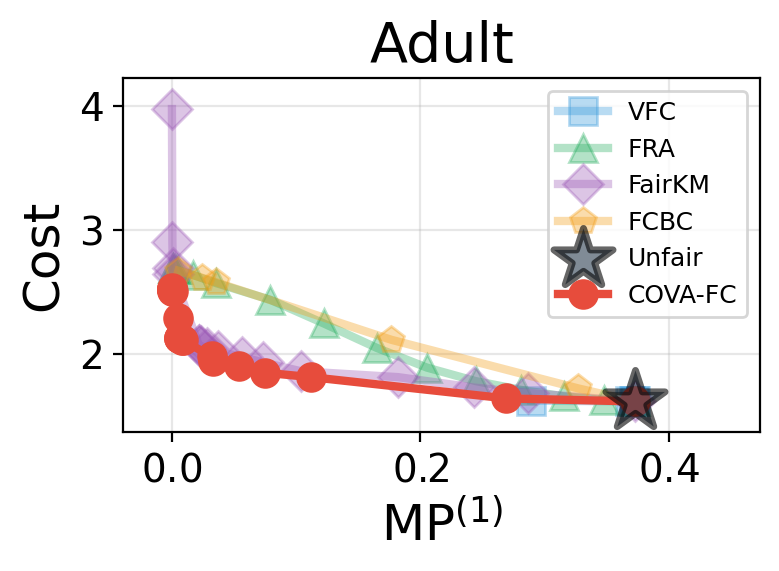} 
    \includegraphics[width=0.245\linewidth]{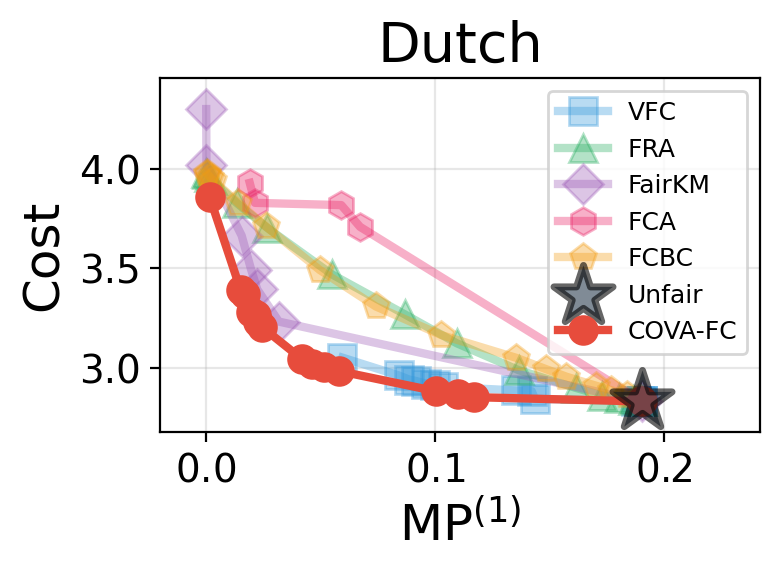}
    \includegraphics[width=0.245\linewidth]{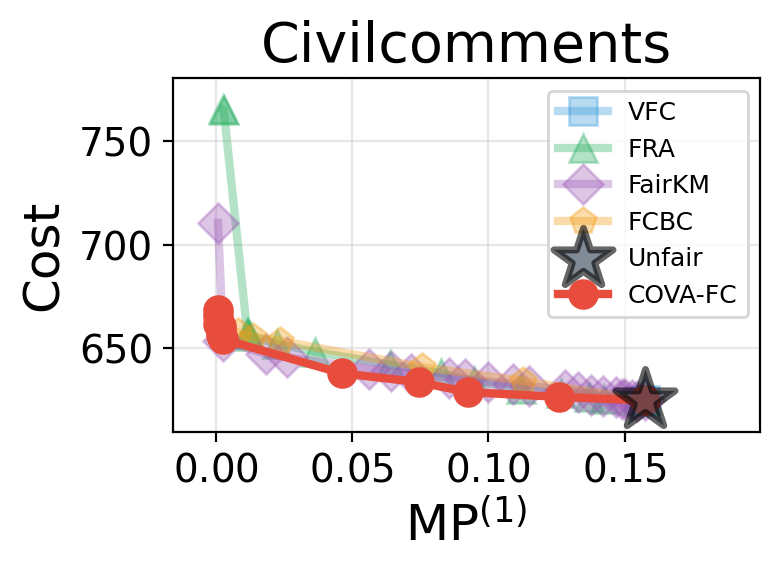}
    \includegraphics[width=0.245\linewidth]{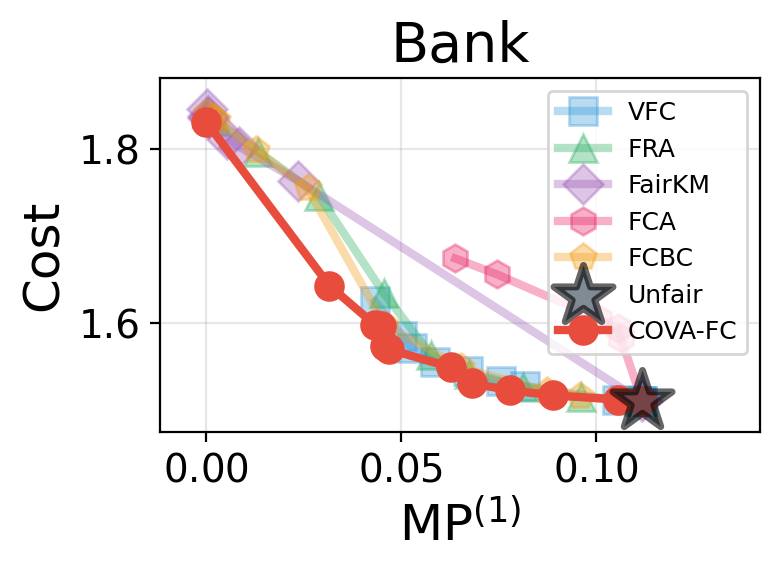}

    \caption{Trade-offs between $\mathrm{MP}^{(1)}$ and cost on (left to right) \texttt{Adult}, \texttt{Dutch}, \texttt{Civilcomments}, and \texttt{Bank} datasets, comparing the subgroup-only and subgroup-marginal variants of \textsc{Cova-FC}.}
    \label{fig:mp_trade_offs}
    % \vskip -0.1in
\end{figure*}

%%%%%%%%%%%%
\subsection{Computational complexity}\label{sec:comp_comp}

\paragraph{Computation time.}

Motivated by the computational complexity analysis in \Cref{tab:complexity}, we compare the runtimes of all methods on \texttt{Dutch} dataset.
As shown in \Cref{tab:runtime}, \textsc{Cova-FC} achieves the lowest runtime among all methods, 
consistent with the theoretical complexity shown in \Cref{tab:complexity}.

% \begin{wraptable}{r}{0.45\textwidth}
\begin{table}[h!]
    \centering
    \caption{Comparison of computation time on \texttt{Dutch} dataset.
    }
    % \vskip 0.1in
    \footnotesize
    \begin{tabular}{lc}
    \toprule
    Method & Total time (sec) \\
    \midrule
    \textbf{\textsc{Cova-FC} $\checkmark$} &  \textbf{30{.}24} \\
    VFC \citep{ziko2021variational} & 48{.}25  \\
    FairKM \citep{abraham2020fairkm} & 292{.}26 \\
    FCA \citep{kim2025fair} & 226{.}42  \\ % 624.21(50 iter)
    FCBC \citep{esmaeili2021fcbc} & 179{.}26 \\
    FRA \citep{bera2019fair} &  64{.}29 \\
    \bottomrule
    \end{tabular}
    \label{tab:runtime}
    \vskip -0.1in
\end{table}
% \end{wraptable}

\paragraph{Scalability to large-scale and image data.}
On \texttt{ACSIncome} \citep{10.5555/3540261.3540757} ($n = 1{,}664{,}500$, 16 subgroups), \textsc{Cova-FC} runs $3.22\times$ faster than VFC, the fastest baseline in \Cref{tab:runtime}. 
% On \texttt{ACSIncome} VFC~\citep{ziko2021variational} also saturates due to numerical overflow (see \Cref{tab:runtime_acs} in Appendix \ref{app:balance-sum-max-gap}). 
A detailed comparison is reported in \Cref{tab:runtime_acs} in Appendix \ref{app:scalability}.
% To verify the applicability of \textsc{Cova-FC} to data with different modalities, we apply \textsc{Cova-FC} to the image dataset \texttt{CelebA} \citep{liu2015faceattributes}.
We use $512$-dimensional features extracted from a pretrained ResNet-18 \citep{he2016deep}, and apply \textsc{Cova-FC} on the fixed features without modification.
As shown in \Cref{fig:celeba_tradeoff}, \textsc{Cova-FC} still improves subgroup fairness as the penalty increases, supporting its applicability across various data modalities.

%%%%%%%%%%%%%%%%%%%%%%%%
\subsection{Effects of the continuous relaxations}\label{sec:cr-comparison}

To examine the effect of each continuous relaxation, we consider four variants of the surrogate depending on which discrete component is continuously relaxed: $\textup{CR}$ (both hard), $\textup{CR}_{\mathbf v}$ (soft $\mathbf v$, hard $\mathbf A$), $\textup{CR}_{\mathbf A}$ (hard $\mathbf v$, soft $\mathbf A$), and $\overline{\textup{CR}}$ (both soft).

We observe that relaxing the assignment \(\mathbf A\) is crucial for stability: hard-assignment variants often fail to reduce the fairness gap, whereas \(\textup{CR}_{\mathbf A}\) achieves near-zero gaps by smoothing the objective $\textup{CR}$ (see \Cref{tab:abl_A_all_tilde}).
This is also reflected in the loss curves, where \(\textup{CR}_{\mathbf v}\) oscillates but \(\overline{\textup{CR}}\) decreases steadily (see \Cref{fig:abl_soft_v_hard_A,fig:app_abl_soft_v_hard_A}).
Relaxing the worst-subgroup selection improves efficiency by replacing \(\textup{CR}_{\mathbf A}\) with \(\overline{\textup{CR}}\), reducing runtime while preserving comparable cost and SP (see \Cref{tab:abl_A_all}).

%%%%%%%%%%%%%%%%%%%%%%%%%%%%%%%%%%%%
%%%%%%%%%%%%%%%%%%%%%%%%%%%%%%%%%%%%
\section{Concluding remarks}

In this paper, we proposed \textsc{Cova-FC}, a novel algorithm for subgroup-fair clustering with multiple sensitive attributes.
Our covariance-based surrogate exactly matches the subgroup-fairness gap and admits efficient gradient-based optimization with parallel assignment updates.
The framework additionally extends to marginal and higher-order subgroups via a user-specified collection $\mathcal W$.
Across benchmark datasets, \textsc{Cova-FC} achieves favorable cost-fairness trade-offs with improved computational efficiency and scalability, when compared with existing baselines.

It would be interesting to extend our proposed framework in several ways. 
First, the framework would be naturally extendable for clustering paradigms other than $K$-means such as hierarchical clustering and mixture models, 
where the cluster assignment matrix $\mathbf A$ admits a similar soft relaxation. 
Second, the covariance-based surrogate can be applied to constraints other than fairness: any constraint expressible as ``low statistical dependence between $\mathbf A$ and a side-information vector''
fits the same template. 
An example would be clustering with size constraints \cite{ZHU2010883, 2008hoppner}.
We leave a systematic study of these extensions as future work.

% Authors are \textbf{required} to include a statement of the potential broader
% impact of their work, including its ethical aspects and future societal
% consequences. This statement should be in an unnumbered section at the end of
% the paper (co-located with Acknowledgements -- the two may appear in either
% order, but both must be before References), and does not count toward the paper
% page limit. In many cases, where the ethical impacts and expected societal
% implications are those that are well established when advancing the field of
% Machine Learning, substantial discussion is not required, and a simple
% statement such as the following will suffice:

% ``This paper presents work whose goal is to advance the field of Machine
% Learning. There are many potential societal consequences of our work, none
% which we feel must be specifically highlighted here.''

% The above statement can be used verbatim in such cases, but we encourage
% authors to think about whether there is content which does warrant further
% discussion, as this statement will be apparent if the paper is later flagged
% for ethics review.


\bibliography{bibs}
\bibliographystyle{plainnat}

%%%%%%%%%%%%%%%%%%%%%%%%%%%%%%%%%%%%%%%%%%%%%%%%%%%%%%%%%%%%%%%%%%%%%%%%%%%%%%%
%%%%%%%%%%%%%%%%%%%%%%%%%%%%%%%%%%%%%%%%%%%%%%%%%%%%%%%%%%%%%%%%%%%%%%%%%%%%%%%
% APPENDIX
%%%%%%%%%%%%%%%%%%%%%%%%%%%%%%%%%%%%%%%%%%%%%%%%%%%%%%%%%%%%%%%%%%%%%%%%%%%%%%%
%%%%%%%%%%%%%%%%%%%%%%%%%%%%%%%%%%%%%%%%%%%%%%%%%%%%%%%%%%%%%%%%%%%%%%%%%%%%%%%
\newpage
\appendix
\onecolumn

\crefalias{section}{appsec}
\crefalias{subsection}{appsubsec}
\crefalias{subsubsection}{appsubsubsec}

%%%%%%%%%%%%%
\section{Two motivating examples}
\subsection{Necessity of the weight $\pi_s$}
\label{sec:weight_necessity}

We illustrate why the subgroup-size weight $\pi_s$ is necessary in the subgroup-fairness gap.
The issue is that, without this weight, a tiny subgroup can dominate the worst-case gap regardless of the overall clustering quality.

\paragraph{Counter example.}
Consider two clusters and suppose that one subgroup $s^\star$ contains only one instance.
Assume that this instance is assigned to cluster 1.
Then the cluster distribution within $s^\star$ is $(1,0)$.

Now suppose that the overall cluster distribution is balanced, namely $(0.5,0.5)$.
The unweighted gap for this subgroup is the total absolute difference between these two distributions:
\[
\sum_{k=1}^{2} \bigl|p_k(\mathbf{A})-p_{k\mid s^{\star}}(\mathbf{A}) \bigr| = |0.5-1|+|0.5-0|=1.
\]
Thus, even a single-instance subgroup can produce a large worst-case gap.

With the weighted gap used in our definition, the same discrepancy is multiplied by the subgroup proportion $\pi_{s^\star}=1/n$.
Hence the contribution of this subgroup becomes
\[
\pi_{s^*} \sum_{k=1}^{2} \bigl|p_k(\mathbf{A})-p_{k\mid s^{\star}}(\mathbf{A}) \bigr| = \frac{1}{n}\times 1=\frac{1}{n}.
\]
Therefore, the weighted gap prevents a tiny subgroup from determining the whole fairness objective.
Instead, the effect of each subgroup is scaled by its empirical size, which makes the objective more stable when many small subgroups exist.

\subsection{Subgroup fairness does not imply marginal fairness}
\label{sec:subgroup_marginal_counterexample}
We give a simple example showing that a small subgroup-fairness gap can coexist with a much larger marginal fairness gap.

Let \(q=8\), \(K=2\), and \(n=2^8=256\). 
Assume that there is exactly one instance in each full subgroup \(s \in \{0,1\}^8\), so \(\pi_s=1/256\) for every \(s\). 
Let \(W=\{s \in \{0,1\}^8 : s^{(1)}=0\}\), which is the marginal subgroup defined by fixing the first sensitive attribute to \(0\). 
Then \(W\) contains \(2^7=128\) subgroups and has proportion \(\pi_W=1/2\). 
Now assign every instance in \(W\) to cluster \(1\), and every instance outside \(W\) to cluster \(2\). 
The overall cluster proportions are therefore \((p_1(\mathbf A),p_2(\mathbf A))=(1/2,1/2)\).

For each full subgroup \(s\), the conditional cluster distribution is either \((1,0)\) or \((0,1)\). 
Hence the subgroup-fairness gap is
\begin{equation}
\Delta(\mathbf A)
=
\frac{1}{256}\left(\left|\frac12-1\right|+\left|\frac12-0\right|\right)
=
\frac{1}{256}.
\end{equation}
By contrast, within the marginal subgroup \(W\), all instances are assigned to cluster \(1\), so \((p_{1\mid W}(\mathbf A),p_{2\mid W}(\mathbf A))=(1,0)\). 
Therefore the marginal fairness gap on \(W\) is
\begin{equation}
\pi_W \sum_{k=1}^2 |p_k(\mathbf A)-p_{k\mid W}(\mathbf A)|
=
\frac12\left(\left|\frac12-1\right|+\left|\frac12-0\right|\right)
=
\frac12.
\end{equation}

Thus the subgroup-fairness gap is only \(1/256\), while the marginal fairness gap is \(1/2\). 
More generally, under the same construction with arbitrary \(q\), the subgroup-fairness gap is \(1/2^q\), which goes to \(0\) as \(q\) increases, whereas the marginal fairness gap remains \(1/2\). 
This shows that subgroup fairness alone does not imply marginal fairness.

\newpage
\section{Theoretical studies}\label{sec:appen-theory}

%%%%%%%%%%%%
\subsection{Proofs}\label{sec:appen-proofs}

\begin{proof}[Proof of \Cref{thm:Delta_CR_equal}]
Take \(\mathcal W=\{\{s\}:s\in\{0,1\}^q,\ n_s>0\}\) in
\Cref{thm:Delta_ms_equiv_hard}.
Then each \(W_m\) is a singleton subgroup, so
\(\Delta(\mathbf A;\mathcal W)=\Delta(\mathbf A)\) and
\(\textup{CR}(\mathbf A;\mathcal W)=\textup{CR}(\mathbf A)\).
\end{proof}

\begin{proof}[Proof of \Cref{thm:Delta_CRbar_equal}]
Take \(\mathcal W=\{\{s\}:s\in\{0,1\}^q,\ n_s>0\}\) in
\Cref{thm:Delta_ms_equiv_hard}.
Then each \(W_m\) is a singleton subgroup, so
\(\Delta(\mathbf A;\mathcal W)=\Delta(\mathbf A)\) and
\(\overline{\textup{CR}}(\mathbf A;\mathcal W)=\overline{\textup{CR}}(\mathbf A)\).
\end{proof}

\begin{proof}[Proof of \Cref{thm:Delta_ms_equiv_hard}]
% Fix $m \in [M]$ and define the subgroup indicator $r_m(i) := \mathbb I(s_i \in W_m),\ i \in [n].$
For each $m \in \mathcal{M}_{+}$, define the subgroup indicator $r_m(i) := \mathbb I(s_i \in W_m),\ i \in [n]$, 
and similarly to $\textup{C}(\mathbf A, s, \beta)$ in \Cref{sec:proposed_pen_cr}, let
\begin{equation}\label{eq:C_Wm_def}
    \textup{C}(\mathbf A, m, \beta) 
    := \frac{1}{2n}\sum_{i=1}^n \left(\sum_{k=1}^K \beta_k \mathbb I(A_i = k)\right)(c_{im} - \bar c_m).
\end{equation}
Let $n_{W_m} = \sum_{i=1}^n r_m(i)$ and recall $\pi_m = n_{W_m}/n.$
For each cluster $k \in [K]$ with $|\{i : A_i = k\}| > 0$ (which holds by our standing assumption that no cluster is empty), define
$
q_k(m) := \frac{1}{|\{i:A_i=k\}|}\sum_{i:A_i=k} r_m(i).
$
A direct calculation gives
$
p_{k \mid W_m}(\mathbf A) = \frac{p_k(\mathbf A) q_k(m)}{\pi_m},
$
so
$
p_k(\mathbf A) - p_{k\mid W_m}(\mathbf A) = \frac{1}{\pi_m} p_k(\mathbf A) (\pi_m - q_k(m)).
$
Multiplying by $\pi_m$ and using $r_m(i) - \pi_m = \tfrac{1}{2}(c_{im} - \bar c_m)$ gives
$$
\pi_m \big(p_k(\mathbf A) - p_{k\mid W_m}(\mathbf A)\big)
= - p_k(\mathbf A) (q_k(m) - \pi_m)
= -\frac{1}{2n}\sum_{i=1}^n \mathbb I(A_i = k)(c_{im} - \bar c_m).
$$
Hence, letting $a_k(m) := \frac{1}{2n}\sum_{i=1}^n \mathbb I(A_i=k)(c_{im}-\bar c_m),$
$$
\pi_m \sum_{k=1}^K |p_k - p_{k\mid W_m}|
= \sum_{k=1}^K |a_k(m)|
= \max_{\beta \in \mathcal B_\infty} \sum_{k=1}^K \beta_k a_k(m)
= \max_{\beta \in \mathcal B_\infty} \textup{C}(\mathbf A, m, \beta),
$$
where the second equality uses the dual representation $\sum_k |a_k| = \max_{\|\beta\|_\infty \le 1} \sum_k \beta_k a_k$ (with maximizer $\beta_k = \operatorname{sign}(a_k(m))$).
Taking $\max_{m \in \mathcal{M}_{+}}$ on both sides and using \Cref{eq:gap-W} and \Cref{eq:CR_W_def} yields $\Delta(\mathbf A; \mathcal W) = \textup{CR}(\mathbf A; \mathcal W).$
\end{proof}

\begin{proof}[Proof of \Cref{thm:Delta_ms_equiv_hard} (soft assignment case)]
The argument mirrors the hard-assignment case with $\mathbb I(A_i = k)$ replaced by $A_{ik}.$
Fix $m \in \mathcal{M}_{+}.$
With $r_m(i) := \mathbb I(s_i \in W_m),$ define
$
q_k(m) := \frac{\sum_i A_{ik} r_m(i)}{\sum_i A_{ik}}
$
(well-defined under the standing assumption $\sum_i A_{ik} > 0$).
Then $p_{k\mid W_m}(\mathbf A) = \frac{p_k(\mathbf A) q_k(m)}{\pi_m},$ so the same algebra used in the hard case gives
$$
\pi_m \sum_{k=1}^K |p_k - p_{k \mid W_m}|
= \max_{\beta \in \mathcal B_\infty} g_m(\beta),
\quad\text{where }
g_m(\beta) := \frac{1}{2n}\sum_{i=1}^n\Big(\sum_{k=1}^K \beta_k A_{ik}\Big)(c_{im} - \bar c_m).
$$
Taking $\max_m$:
\begin{equation}\label{eq:Delta_repr_soft}
    \Delta(\mathbf A; \mathcal W) = \max_{m \in \mathcal{M}_{+}} \max_{\beta \in \mathcal B_\infty} g_m(\beta).
\end{equation}
For the right-hand side of \Cref{eq:Delta_repr_soft}, observe that for any fixed $\beta,$ the function $\mathbf v \mapsto \mathbf v^\top g(\beta)$ is linear on the simplex $\mathbb S^{M_{+}}$ and so attains its maximum at a vertex.
Therefore $\max_{\mathbf v \in \mathbb S^{M_{+}}} \mathbf v^\top g(\beta) = \max_{m \in \mathcal{M}_{+}} g_m(\beta),$ and
$$
\overline{\textup{CR}}(\mathbf A; \mathcal W)
= \max_{\beta \in \mathcal B_\infty} \max_{\mathbf v \in \mathbb S^{M_{+}}} \mathbf v^\top g(\beta)
= \max_{\beta \in \mathcal B_\infty} \max_{m \in \mathcal{M}_{+}} g_m(\beta)
= \max_{m \in \mathcal{M}_{+}} \max_{\beta \in \mathcal B_\infty} g_m(\beta).
$$
Combining with \Cref{eq:Delta_repr_soft} gives $\Delta(\mathbf A; \mathcal W) = \overline{\textup{CR}}(\mathbf A; \mathcal W).$
\end{proof}

% %%%%%%%%
\clearpage
\subsection{Projection of $(A_{i1}, \ldots, A_{iK})$ onto the simplex and the computation of $\tau_i$}\label{sec:appen-simplex-proj}

For each $i\in[n]$, after a gradient step we obtain ${\mathbf A}^{s}_{i}\in\mathbb R^{K}$ and project it onto the probability simplex $\mathbb S^{K}$.
The projection is defined as the unique solution to the strictly convex problem
\begin{equation}\label{eq:simplex_proj_problem}
    \Pi_{\mathbb S^{K}}({\mathbf A}^{s}_{i})
    = \argmin_{v \in \mathbb S^{K}}\frac12\|v-{\mathbf A}^{s}_{i}\|^2.
\end{equation}
By the KKT conditions (see \cite{wang2013projectionprobabilitysimplexefficient}), there exists a scalar threshold $\tau_i\in\mathbb R$ such that the optimizer satisfies
\begin{equation}\label{eq:simplex_proj_threshold}
    v^\star_{ik}=\max\{ A^{s}_{ik}-\tau_i,\,0\},\quad k\in[K],
\end{equation}
and $\tau_i$ is uniquely determined by the constraint $\sum_{k=1}^{K} v^\star_{ik}=1$.

To compute $\tau_i$ efficiently, let $u_1\ge u_2\ge\cdots\ge u_K$ be the sorted entries of ${\mathbf A}^{s}_{i}$, and define the cumulative sums
$U_j:=\sum_{\ell=1}^{j}u_\ell$ for $j\in[K]$.
Let
\begin{equation}\label{eq:tau_derive}
    \rho:=\max\left\{j\in[K]: u_j-\frac{1}{j}(U_j-1)>0\right\},
    \quad
    \tau_i:=\frac{1}{\rho}(U_\rho-1).  
\end{equation}
Then \Cref{eq:simplex_proj_threshold} is equal to $\Pi_{\mathbb S^{K}}({\mathbf A}^{s}_{i})$.
Finally, we project as
$$
A_{ik}^{s} \leftarrow \max \{ A_{ik}^{s} - \tau_{i}, 0 \},
\quad i\in[n],\ k\in[K],
$$
where each $\tau_{i}$ is determined in \Cref{eq:tau_derive}.
The computational complexity of this step is dominated by sorting, i.e., $O(K\log K)$ for each $i\in[n]$, which is negligible compared to the other computations in \textsc{Cova-FC} (see \Cref{tab:complexity}).

%%%%%%%%%%%%%%%%%%%%%%%%%
\newpage

%%%%%%%%%%%%%%%%%%%%%%%%%%%%%%%%%%%%%%
\section{Detailed Updates of \textsc{Cova-FC}}\label{app:cova_fc_algorithm}
Given a collection $\mathcal W$ of subgroup-subsets, we propose to learn fair clusters by minimizing \Cref{eq:obj}: $\frac{1}{n}\sum_{i=1}^n L_i(\mathbf \mu, \mathbf A) + \lambda\, \overline{\textup{CR}}(\mathbf A; \mathcal W)$
% \begin{equation}\label{eq:obj}
%     \frac{1}{n}\sum_{i=1}^n L_i(\mathbf \mu, \mathbf A) + \lambda\, \overline{\textup{CR}}(\mathbf A; \mathcal W)
% \end{equation}
with respect to $\mathbf \mu = (\mu_1, \ldots, \mu_K) \in (\mathbb R^d)^K$ and $\mathbf A = (A_i)_{i=1}^n,\ A_i \in \mathbb S^K,$ where $\lambda \ge 0$ controls the trade-off between clustering cost and fairness level.
We call this algorithm \textit{\textbf{Cova}riance-based \textbf{F}air \textbf{C}lustering (\textsc{Cova-FC})}.
We solve \Cref{eq:obj} via an adversarial learning procedure.
First, we initialize the centroids $\mathbf \mu$ and the assignment $\mathbf A$ using the output of the standard $K$-means algorithm \citep{1056489}.
Define
$
\overline{\textup{C}}(\mathbf A, \mathbf v, \beta) := \frac{1}{2n}\sum_{i=1}^n \left(\sum_{k=1}^K \beta_k A_{ik}\right) \mathbf v^\top (c_{i} - \bar c),
$
so that $\overline{\textup{CR}}(\mathbf A; \mathcal W) = \max_{\mathbf v \in \mathbb S^{M_{+}}} \max_{\beta \in \mathcal B_\infty} \overline{\textup{C}}(\mathbf A, \mathbf v, \beta).$
We iteratively (i) maximize $\overline{\textup{C}}(\mathbf A, \mathbf v, \beta)$ over $(\mathbf v, \beta)$ with $\mathbf A$ fixed, and (ii) update $\mathbf A$ by minimizing $\frac{1}{n}\sum_i L_i(\mathbf \mu, \mathbf A) + \lambda\, \overline{\textup{C}}(\mathbf A, \mathbf v, \beta)$ with $(\beta, \mathbf v, \mathbf \mu)$ fixed, followed by updating $\mathbf \mu$ given $\mathbf A.$

%%%%%%%%%%%%%%%%%%%
\paragraph{Updates of $\beta$ and $\mathbf v$.}

For fixed $\mathbf A,$ we alternately update $\beta$ (with $\mathbf v$ fixed) and $\mathbf v$ (with $\beta$ fixed).
We first maximize $\overline{\textup{C}}(\mathbf A, \mathbf v, \beta)$ over $\beta$ by the closed-form update
$
\beta_k \leftarrow \operatorname{sign}\left(\sum_{i=1}^n A_{ik}\, \mathbf v^\top (c_i - \bar c)\right),
$
which gives
$
\max_{\beta \in \mathcal B_\infty}\sum_i \sum_k \beta_k A_{ik} \mathbf v^\top (c_i - \bar c)
= \sum_k |\sum_i A_{ik}\, \mathbf v^\top (c_i - \bar c)|.
$
With this updated $\beta$ fixed, we update $\mathbf v$ via gradient ascent using the closed-form gradient
$
\frac{1}{2n}\sum_k \beta_k \left(\sum_i A_{ik}\, (c_i - \bar c)\right),
$
and project onto $\mathbb S^{M_{+}}$ using projection for $\mathbf A$ (see \Cref{eq:proj_A}).

%%%%%%%%%%%%%%%%%%%
\paragraph{Updates of $\mathbf A$ and $\mathbf \mu$.}

With $\beta$ and $\mathbf v$ fixed, we update $\mathbf A$ in the direction of minimizing
\begin{equation}\label{eq:adv_min}
    \frac{1}{n}\sum_{i=1}^n L_i(\mathbf \mu, \mathbf A) + \lambda\, \overline{\textup{C}}(\mathbf A, \mathbf v, \beta),
\end{equation}
whose gradient with respect to $A_{ik}$ is
\begin{equation}\label{eq:grad_Aik}
    \frac{\partial \textup{\Cref{eq:adv_min}}}{\partial A_{ik}}
    = \frac{1}{n}\|x_i - \mu_k\|^2 + \frac{\lambda}{2n}\beta_k\, \mathbf v^\top (c_i - \bar c).
\end{equation}
We update $A_{ik}$ by a gradient-descent step using \Cref{eq:grad_Aik}, then project the vector $(A_{i1}, \ldots, A_{iK})^\top$ onto the simplex $\mathbb S^K$ via
\begin{equation}\label{eq:proj_A}
    A_{ik} \leftarrow \max\{A_{ik} - \tau_i, 0\}, \quad k \in [K],
\end{equation}
where $\tau_i$ is chosen so that $\sum_{k=1}^K A_{ik} = 1$ \citep{wang2013projectionprobabilitysimplexefficient}.
The explicit calculation of $\tau_i$ is given in \Cref{sec:appen-simplex-proj}.
Since the gradient in \Cref{eq:grad_Aik} is in closed form, this step is computationally cheap.

We then fix $\mathbf A$ and update centroids as in standard $K$-means: $\mu_k \leftarrow \sum_{i=1}^n A_{ik} x_i / \sum_{i=1}^n A_{ik}.$
The overall algorithm is summarized in \Cref{alg:corefc_fullbatch}.

%%%%%%%%%%%
\begin{algorithm}[h!]
    \caption{\textsc{Cova-FC}}
    \label{alg:corefc_fullbatch}
    \begin{algorithmic}[1]
        \Require $\{(x_i, s_i)\}_{i=1}^n,\ \mathcal W = \{W_m\}_{m=1}^M,\ K,\ \lambda$
        \Require Step sizes $\eta_{\mathbf A}$ and $\eta_{\mathbf v}$
        \State Initialize $\mathbf A$ and $\mathbf \mu$ by the outputs of the standard $K$-means algorithm
        \State Compute $c_{im} = 2\mathbb I(s_i \in W_m) - 1$ for $m \in [M], i \in [n]$,   
        set $c_i = (c_{im})_{m \in \mathcal{M}_+}$, and $\bar c \leftarrow \frac{1}{n}\sum_{i=1}^n c_i$
        \State Randomly initialize $\mathbf v \in \mathbb R^{M_{+}}$ then project it onto the simplex $\mathbb S^{M_+}$
        \While{not converged}
            \State \textbf{Adversarial step}
            \State \quad $\beta_k \leftarrow \operatorname{sign}\left(\sum_{i=1}^n A_{ik}\, \mathbf v^\top (c_i - \bar c)\right),\ k \in [K]$
            \State \quad $\mathbf v \leftarrow \mathbf v + \eta_{\mathbf v} \cdot \frac{1}{2n}\sum_{k=1}^K \beta_k \left(\sum_{i=1}^n A_{ik}\, (c_i - \bar c)\right)$ and Project $\mathbf v$ onto the simplex $\mathbb S^{M_{+}}$
            % \State \quad Project $\mathbf v$ onto the simplex $\mathbb S^M$
            \State \textbf{Objective minimization step}
            \State \quad $A_{ik} \leftarrow A_{ik} - \eta_{\mathbf A}\left(\frac{1}{n}\|x_i - \mu_k\|^2 + \frac{\lambda}{2n}\beta_k\, \mathbf v^\top (c_i - \bar c)\right)$
            \State \quad Project $(A_{ik})_{k=1}^K$ onto the simplex following \Cref{eq:proj_A}
            \State \quad $\mu_k \leftarrow \sum_{i=1}^n A_{ik} x_i / \sum_{i=1}^n A_{ik}$
        \EndWhile
        \State \textbf{Return} $\mathbf A,\ \mathbf \mu$
    \end{algorithmic}
\end{algorithm}

%%%%%%%%%%%%%%%%%%%
\clearpage
\subsection{Computational complexity}\label{app:complexity}

An additional advantage of using \(\overline{\textup{CR}}\) as the fairness penalty is that it offers competitive computational complexity compared to existing fair clustering algorithms.
Let \(M\) be the number of sensitive groups included in the fairness penalty.
Each iteration of \textsc{Cova-FC} has computational complexity \(\mathcal O\bigl(nKd+n(K+M)\bigr)\).
Here, \(nKd\) is the standard \(K\)-means distance-computation term, while \(n(K+M)\) is the additional cost for the adversarial fairness updates.

In the main subgroup-only setting, \(\mathcal W\) is the collection of full subgroups, and hence \(M=2^q\).
The same complexity expression applies to the optional marginal extension by replacing \(M\) with the size of the chosen subgroup-subset collection.
For example, when the extension uses marginal, second-order, and full subgroups, \(M=2q+4\binom{q}{2}+2^{q}=O(2^q)\).

\begin{table}[h!]
    \centering
    \caption{Comparison of fair clustering methods in terms of computational complexity.}
    \label{tab:appen-complexity}
    \footnotesize
    \setlength{\tabcolsep}{10pt}
    \begin{tabular}{l l}
    \toprule
    Method & Complexity \\
    \midrule
    \textbf{\textsc{Cova-FC}} $\checkmark$ 
    & $n (Kd + K + 2^{q})$ \\
    
    VFC \citep{ziko2021variational}
    & $nK (d + 2^{q})$ \\
    
    FairKM \citep{abraham2020fairkm}
    & $nK(nd + 2^q)$ \\ 
    
    FCA \citep{kim2025fair}
    & $n^2 (n + Kd)$ \\
    
    FCBC \citep{esmaeili2021fcbc}
    & $(nK+2^{q}K)^3$ \\
    
    FRA \citep{bera2019fair}
    & $(nK+2^{q})^3$ 
    \\
    \bottomrule
    \end{tabular}
    % \vskip -0.1in
\end{table}

In contrast, most of the fair clustering algorithms have higher computational complexity:
(i) VFC \citep{ziko2021variational} requires computational complexity linear in $n$ similarly to \textsc{Cova-FC}
but the additional term due to the fairness constraint is $n K 2^q$ which becomes much larger than
the additional term in \textsc{Cova-FC} that is $n (K+2^q)$ when $q$ is large.
(ii) FairKM \citep{abraham2020fairkm} includes a quadratic term $\mathcal O(n^2Kd)$, which is substantially larger than first-order updates for large-scale datasets.
(iii) FCA \citep{kim2025fair}, FCBC \citep{esmaeili2021fcbc}, and FRA \citep{bera2019fair} are dominated by solving LPs, yielding cubic complexity with respect to $n.$

Overall, \textsc{Cova-FC} achieves a favorable computational complexity, making it more suitable for large datasets with multiple sensitive attributes.
We also observe that \textsc{Cova-FC} runs faster than the baselines, whose results are provided in \Cref{tab:runtime} of \Cref{sec:tradeoff_compare}.
Furthermore, \textsc{Cova-FC} runs $3.2 \times$ faster than VFC on a large-scale dataset ($n \approx $ millions, see \Cref{app:scalability} for details).

%%%%%%%%%%%%%%%%%%%%%%%%%
\newpage
\section{Experiments}\label{sec:app-exp}

\subsection{Settings}\label{sec:appen-impl}

In this section, we provide more details about datasets, baseline methods, and implementation details.

\paragraph{Datasets.}
We use seven benchmark datasets: \texttt{Adult}, \texttt{Dutch}, \texttt{Bank}, \texttt{Civilcomments}, \texttt{Communities}, \texttt{ACSIncome}, and \texttt{CelebA}.
See \cref{tab:datasets} for a summary of these datasets.
For tabular datasets, we use numerical (continuous) features only.
We standardize the features of all the datasets to zero mean and unit variance.

\begin{itemize}[noitemsep=1pt, topsep=1pt]
    \item
    {\texttt{Adult}}:
    The Adult dataset is extracted from the 1994 U.S. Census database \citep{becker1996adult}.
    We use $d=5$ continuous features: \texttt{age}, \texttt{fnlwgt}, \texttt{education-num}, \texttt{capital-gain}, and \texttt{hours-per-week}.
    We use four binary sensitive attributes: gender (male/female), race (white/non-white), age ($\ge 40$ / $< 40$), and marital status (married/unmarried),
    so $q=4$ and there are $2^q=16$ subgroups.
    The number of instances is $n=32{,}561$.

    \item
{\texttt{Dutch}}:
The Dutch dataset is extracted from the 2001 Netherlands census and is used following the fairness dataset construction in \cite{lequy2022survey}.
We use $d=9$ numeric features (all census attributes except the sensitive ones) which are standardized.
We use two binary sensitive attributes: gender (male/female) and age (above median / at or below median),
so $q=2$ and there are $2^q=4$ subgroups.
The number of instances is $n=60{,}420$.

    \item
{\texttt{Bank}}:
The Bank dataset contains records from a Portuguese bank's direct marketing campaigns \citep{MORO201422}.
We use $d=6$ continuous features: \texttt{age}, \texttt{call-duration}, \texttt{3-month-Euribor-rate},
\texttt{number-of-employees}, \texttt{consumer-price-index}, and \texttt{number-of-contacts}.
We treat marital status as a sensitive attribute with three categories (single/married/divorced) and exclude \texttt{unknown} entries,
so there are $3$ disjoint subgroups.
The number of instances is $n=41{,}108$.

    \item
{\texttt{Civilcomments}}:
The Civilcomments dataset is a toxicity dataset of online comments annotated with identity attributes \citep{borkan2019nuanced}.
We use $d=768$ feature embeddings extracted from the dataset.
We use three multi-valued sensitive attributes: gender (male/female/other), religion (christian/other), and race (black/white/asian/other),
so there are $3\times 2\times 4 = 24$ subgroups.
The number of instances is $n=3{,}365$.

    \item
{\texttt{Communities}}:
The Communities and Crime dataset records community-level demographic and crime statistics across U.S. communities \citep{redmond2002data}.
We use $d=99$ continuous features after removing entries with missing values.
We use $q=18$ binary sensitive attributes derived from race-related demographic ratios, each binarized at its median.
The Cartesian product yields up to $2^{18}$ combinations, of which $1{,}180$ are non-empty in the data.
The number of instances is $n=1{,}994$.

    \item
{\texttt{ACSIncome}}:
ACSIncome is a large-scale tabular dataset extracted from the American Community Survey \citep{10.5555/3540261.3540757}.
We use $d=7$ continuous features.
We use three sensitive attributes: gender (male/female), marital status (married/unmarried), and race (white/black/asian/other),
so $q=3$ and there are $2 \times 2 \times 4 = 16$ subgroups.
The number of instances is $n=1{,}664{,}500$.

    \item
{\texttt{CelebA}}:
The CelebA dataset contains face images annotated with $40$ binary attributes \citep{liu2015faceattributes}.
We extract $d=512$ image features from a pre-trained ResNet-18 \citep{he2016deep} and apply \textsc{Cova-FC} on the fixed features without modification.
We use four binary sensitive attributes (Male, Young, Smiling, Attractive),
so $q=4$ and there are $2^q=16$ subgroups.
The number of instances is $n=202{,}599$.

\end{itemize}

\begin{table*}[h!]
    \centering
    \caption{
    Summary statistics of datasets.
    $n$ is the number of instances, $d$ is the number of features, and $q$ is the number of sensitive attributes.
    ``\# subgroups per attr'' lists the number of categories for each sensitive attribute
    and ``\# Subgroups'' is the number of subgroups formed by the Cartesian product of sensitive attributes.
    }
    \vskip 0.1in
    \footnotesize
    \label{tab:datasets}
    \setlength{\tabcolsep}{10pt}
    \begin{tabular}{l r r r c r}
    \toprule
    Dataset & $n$ & $d$ & $q$ & \# subgroups per attr & \# Subgroups \\
    \midrule
    \texttt{Adult}         & 32{,}561 & 5   & 4  & $2/2/2/2$ & 16   \\
    \texttt{Dutch}         & 60{,}420 & 9   & 2  & $2/2$     & 4    \\
    \texttt{Bank}          & 41{,}108 & 6   & 1  & $3$       & 3    \\
    \texttt{Civilcomments} & 3{,}365  & 768 & 3  & $3/2/4$   & 24   \\
    \texttt{Communities}   & 1{,}994  & 99  & 18 & $2$ (each) & 1{,}180 \\
    \texttt{ACSIncome}     & 1{,}664{,}500 & 7 & 3 & $2/2/4$ & 16 \\
    \texttt{CelebA}        & 202{,}599 & 512 & 4 & $2/2/2/2$ & 16 \\
    \bottomrule
    \end{tabular}
\end{table*}

\paragraph{Baselines.}

FCA faces a bottleneck: its core step estimates joint distributions over subgroups by solving a linear program, which incurs high computational cost as the number of subgroups and instances increases.
On the larger \texttt{Adult} and \texttt{Civilcomments} datasets, an experiment did not complete within 6 hours, so we omit FCA from these comparisons and report it only on \texttt{Dutch} and \texttt{Bank}.

\paragraph{Implementation details.}

To control the cost-fairness trade-off, we sweep $\lambda$ on a grid:
$\lambda \in [10^{-1}, 10^{4}]$ for \texttt{Adult}, \texttt{Dutch}, and \texttt{Bank}, and
$\lambda \in [10^{-1}, 10^{6}]$ for \texttt{Civilcomments}.
We initialize $\mathbf{\mu}$ and $\mathbf{A}$ using the outputs of the standard $K$-means algorithm \citep{1056489}.
For each iteration, we alternate one update for $(\beta, \mathbf{v})$, multiple updates for $\mathbf{A}$, and then update $\mathbf{\mu}$.
For the $\mathbf{A}$ updates: the learning rate and the number of updates are set to
$(0.004, 100)$ for \texttt{Adult},
$(0.005, 100)$ for \texttt{Dutch},
$(0.005, 100)$ for \texttt{Bank}, and
$(0.007, 50)$ for \texttt{Civilcomments}.
All experiments are conducted on a Linux server with four NVIDIA RTX 3090 GPUs (24GB VRAM each) and a 48-core Intel Xeon Silver 4310 CPU.

%%%%%%%%%%%%%%%%%%%%%%%%%%
% \newpage
\subsection{Performance comparison}\label{app:balance-sum-max-gap}

\paragraph{Cost-fairness trade-off: sum-aggregated gap.}

For a more comprehensive analysis, we further examine whether the same conclusions hold under an alternative aggregation rule.
While the gap $\Delta$ in \eqref{eq:gap-W} takes a worst-case (maximum) over subgroups, we also report sum-aggregated variants $\textup{SP}_{\textup{sum}}$ and $\textup{MP}^{(l)}_{\textup{sum}}$ defined in \Cref{tab:abl_fairness_metrics}.

\begin{table*}[hbt!]
    \centering
    \caption{Sum-aggregated variants of the gap. SP and MP are special cases of $\Delta(\mathbf A; \mathcal W)$ in \eqref{eq:gap-W} with appropriate $\mathcal W$. We follow the notation of \Cref{sec:partial,sec:marg_sub}: $\theta \subset [q]$ with $|\theta|=l$ indexes an $l$-th order marginal subgroup, and $W_{s_\theta, a} = \{s \in \{0,1\}^q : s_\theta = a\}$ for $a \in \{0,1\}^l$.}
    \label{tab:abl_fairness_metrics}
    \vskip 0.1in
    \footnotesize
    \setlength{\tabcolsep}{5pt}
    \renewcommand{\arraystretch}{0.95}
    \begin{tabular}{l
        >{\centering\arraybackslash}p{0.53\textwidth}
        >{\centering\arraybackslash}p{0.39\textwidth}}
        \toprule
        Metric & Base & Aggregation \\
        \midrule
        SP$_{\text{sum}}$
        & $\delta_{k,s} := |p_k(\mathbf A) - p_{k \mid s}(\mathbf A)|$
        & $\sum_{s \in \{0,1\}^q} \pi_s \sum_{k \in [K]}\delta_{k,s}$ \\
        MP$^{(l)}_{\text{sum}}$
        & $\delta_{k,\theta,a} := |p_k(\mathbf A) - p_{k \mid W_{s_\theta, a}}(\mathbf A)|$
        & $\sum_{\theta \subset [q],\ |\theta|=l}\sum_{a \in \{0,1\}^l} \pi_{W_{s_\theta, a}} \sum_{k \in [K]}\delta_{k,\theta,a}$ \\
        \bottomrule
    \end{tabular}
\end{table*}

As shown in \Cref{fig:gap_max_trade_offs}, replacing the max with sum aggregation does not change the main messages of our experiments.
That is, \textsc{Cova-FC} still attains competitive cost-fairness trade-offs, and this indicates that \textsc{Cova-FC} is not limited to specific aggregation rules.

\begin{figure}[h]
    \centering
    \includegraphics[width=0.24\linewidth]{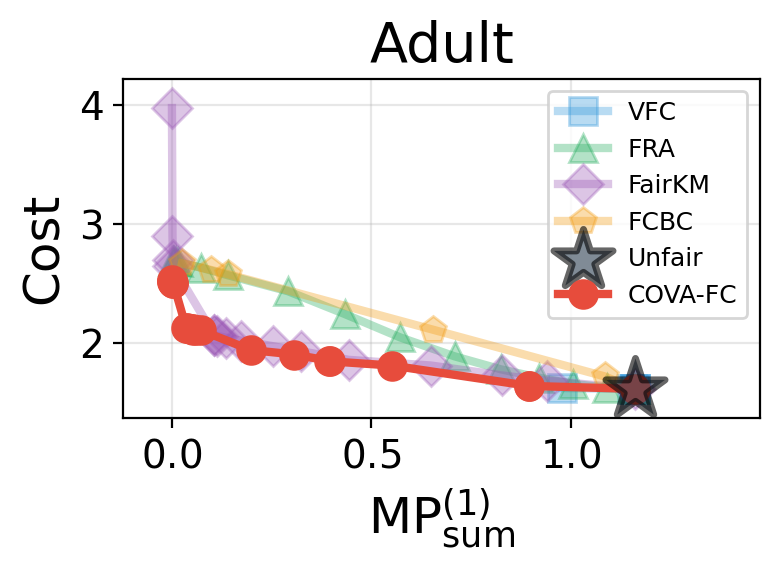}
    \includegraphics[width=0.24\linewidth]{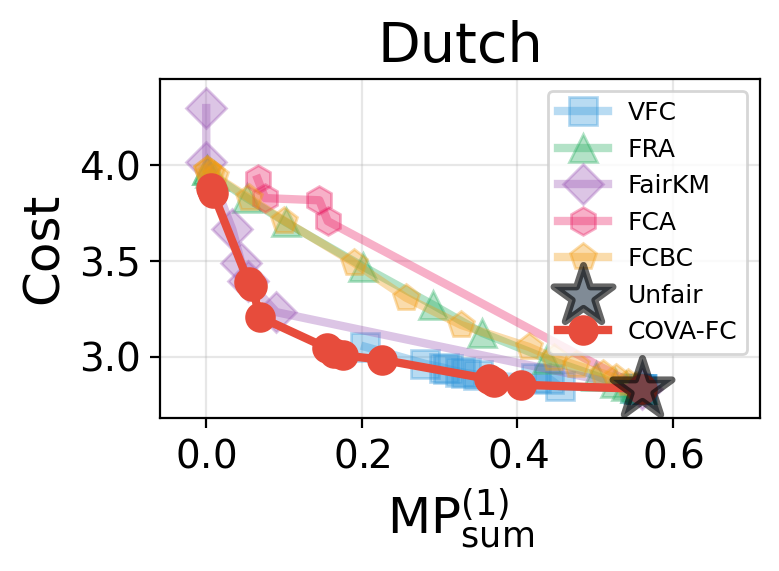}
    \includegraphics[width=0.24\linewidth]{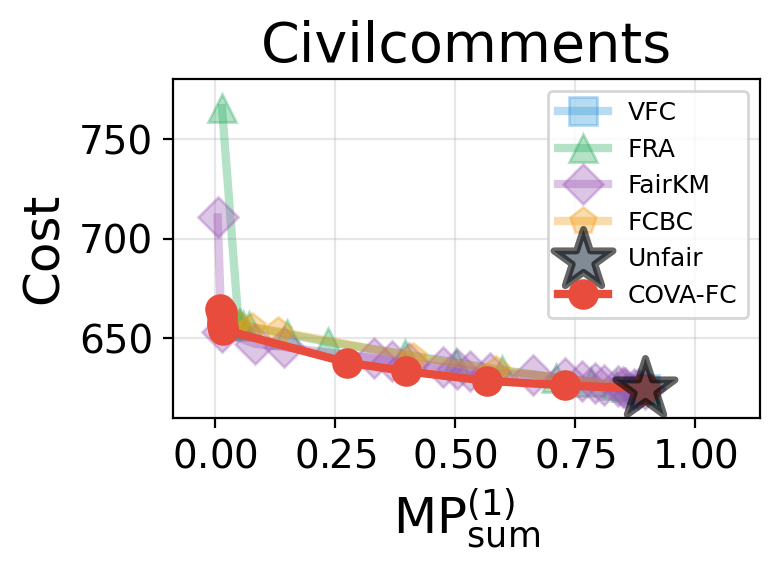}
    \includegraphics[width=0.24\linewidth]{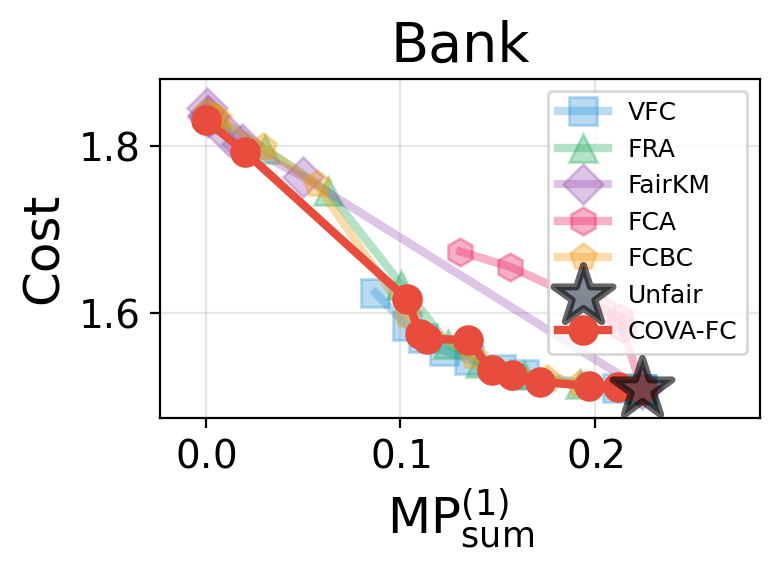}

    \includegraphics[width=0.24\linewidth]{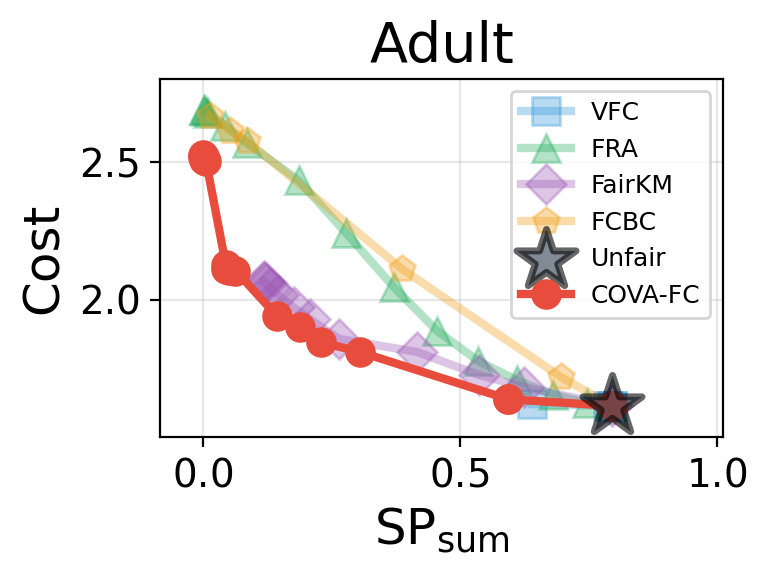}
    \includegraphics[width=0.24\linewidth]{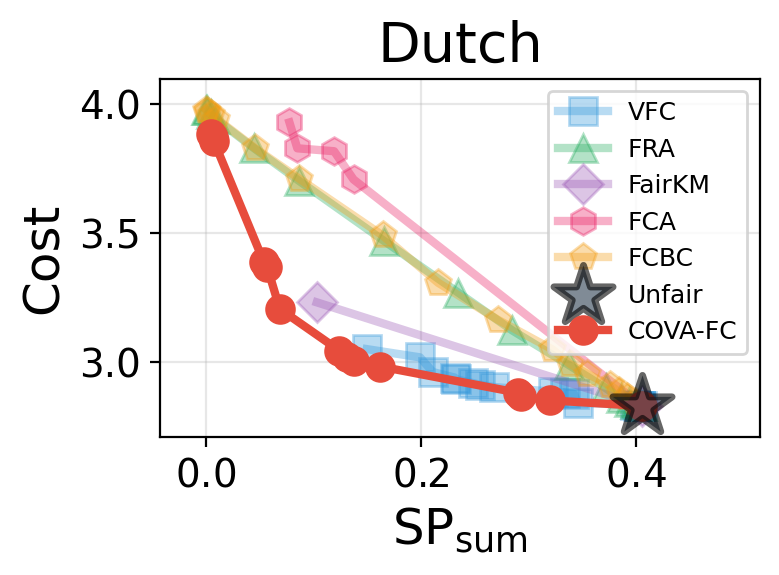}
    \includegraphics[width=0.24\linewidth]{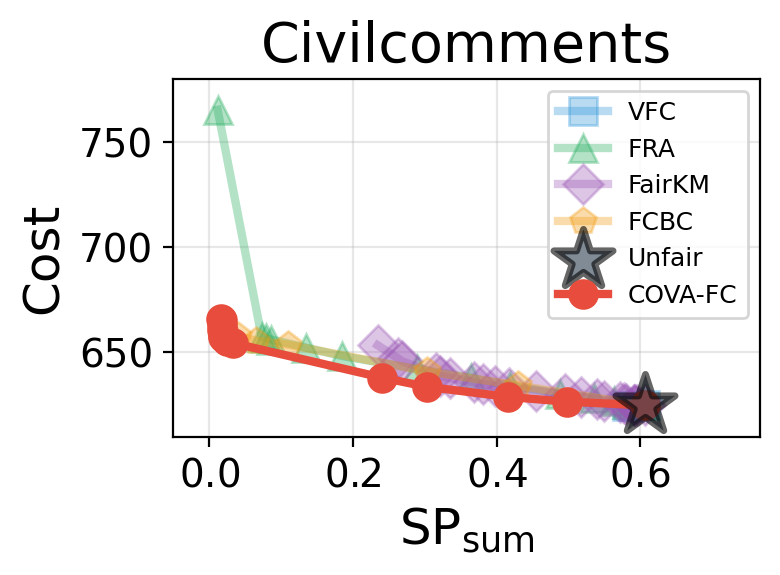}
    \includegraphics[width=0.24\linewidth]{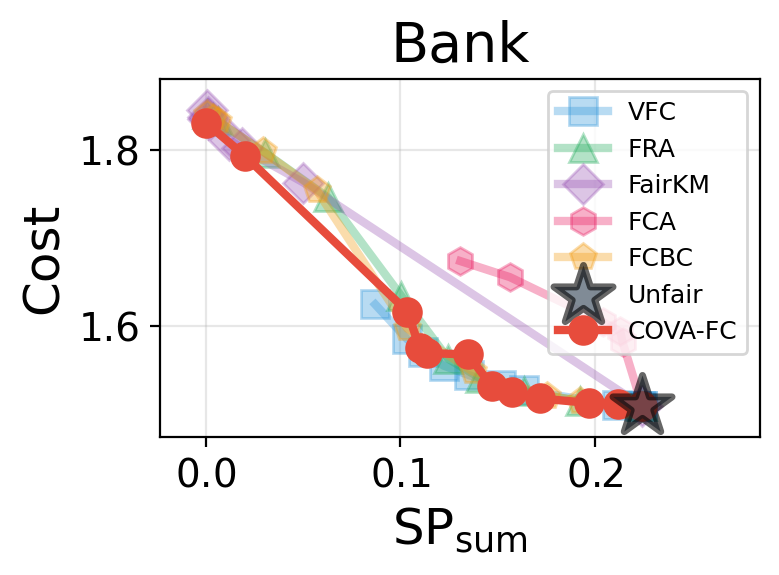}

    \caption{
    Comparison of trade-offs between `sum-aggregated' gap (top: $\textup{MP}^{(1)}_{\textup{sum}}$, bottom: $\textup{SP}_{\textup{sum}}$ ) and cost
    on (left to right) \texttt{Adult}, \texttt{Dutch}, \texttt{Civilcomments}, and \texttt{Bank} datasets.
    Note that $\textup{SP}_{\textup{sum}}$ = $\textup{MP}^{(1)}_{\textup{sum}}$ for \texttt{Bank} dataset since $q=1$.
    }
    \label{fig:gap_max_trade_offs}
    \vskip -0.1in
\end{figure}

%%%%%%%%%%%%%%%
\newpage
\paragraph{Cost-fairness trade-off: subgroup Balance.}

As discussed in the Remark in \Cref{sec:exp-setting}, Balance \citep{chierichetti2017fair} is the standard ratio-based fairness measure but degenerates on datasets where some subgroup has fewer than $K$ instances.
We use a probability-ratio version of Balance, rather than the original count-ratio form, because the latter can depend strongly on subgroup-size ratios even under perfectly balanced cluster proportions.
Using the notation of \Cref{sec:partial,sec:marg_sub}, the subgroup Balance (SB) and the $l$-th order marginal Balance ($\textup{MB}^{(l)}$) are defined as
\[
    \textup{SB}:= \min_{s \neq s'} \min_{k \in [K]} B_{k,s,s'},
    \qquad
    \textup{MB}^{(l)}:= \min_{\theta \subset [q],\, |\theta|=l}\, \min_{a \neq a'}\, \min_{k \in [K]} B_{k,\, W_{s_\theta, a},\, W_{s_\theta, a'}},
\]
where, for any two subgroup-subsets $W, W'$ with $n_W, n_{W'} > 0$ and any $k \in [K]$,
\[
    B_{k,W,W'} := \min\!\left\{
        \frac{p_{k\mid W}(\mathbf A)}{p_{k\mid W'}(\mathbf A)},
        \ \frac{p_{k\mid W'}(\mathbf A)}{p_{k\mid W}(\mathbf A)}
    \right\},
\]
$B_{k,s,s'}$ is the special case $W=\{s\},\ W'=\{s'\}$, and the minimum for $\textup{MB}^{(l)}$ is over all pairs of $l$-th order marginal subgroups.
For datasets where these Balances are well-defined (\texttt{Adult}, \texttt{Dutch}, \texttt{Bank}), we additionally report cost vs.\ SB and $\textup{MB}^{(1)}$ in \Cref{fig:sb_trade_offs}.
\textsc{Cova-FC} attains higher Balance at comparable clustering cost than the baselines in most cases, consistent with the gap-based results.

On \texttt{Adult}, both VFC and FairKM stay at $\textup{SB}=0$, while \textsc{Cova-FC} reaches strictly positive SB. 
Empirically, we observe that FairKM's hard assignments often leave at least one small subgroup empty in some cluster (\(p_{k\mid s}=0\)), and VFC can collapse to the \(K\)-means solution under tiny subgroups due to its inverse subgroup-cluster mass term (\Cref{app:vfc_deep_compare}).
\textsc{Cova-FC}'s soft assignment and inverse-mass-free surrogate keep all $p_{k\mid s}$ strictly positive, so SB stays positive.

\begin{figure}[h]
    \centering
    \includegraphics[width=0.3\linewidth]{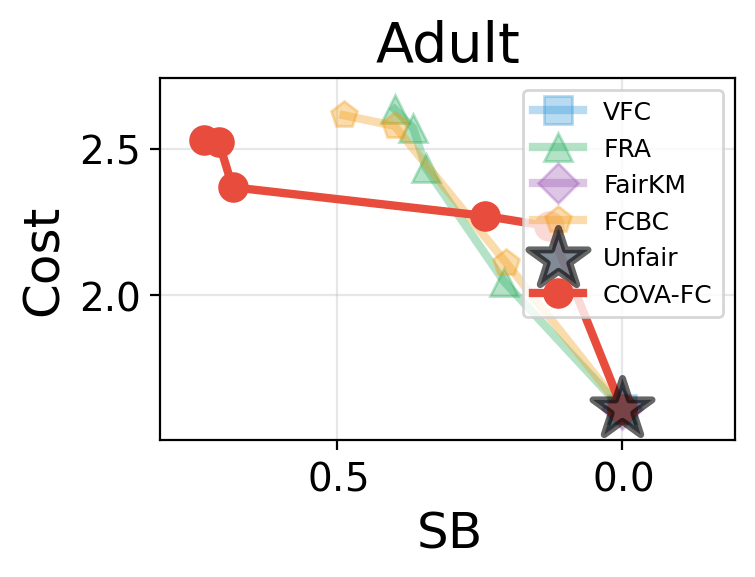}
    \includegraphics[width=0.3\linewidth]{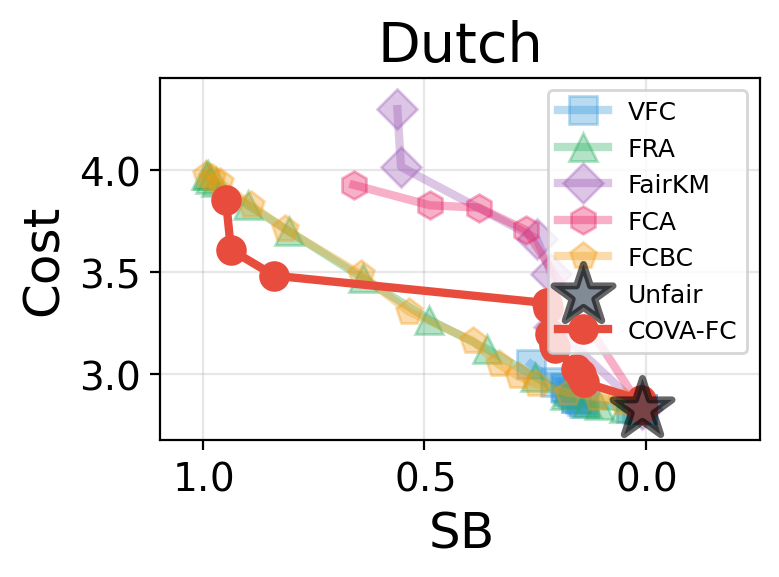}
    \includegraphics[width=0.3\linewidth]{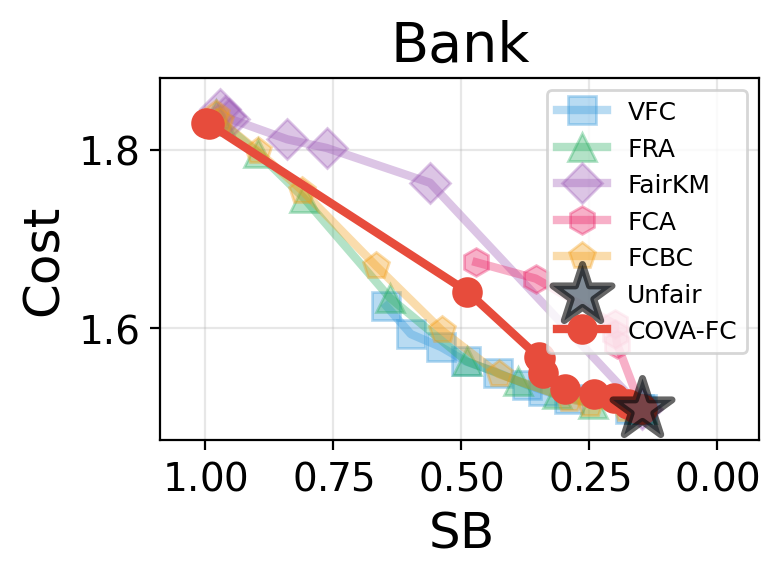}

    \includegraphics[width=0.3\linewidth]{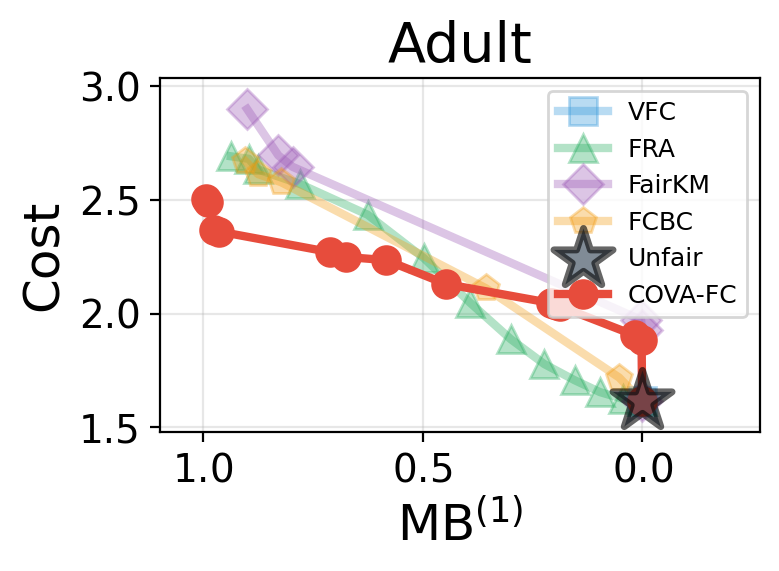}
    \includegraphics[width=0.3\linewidth]{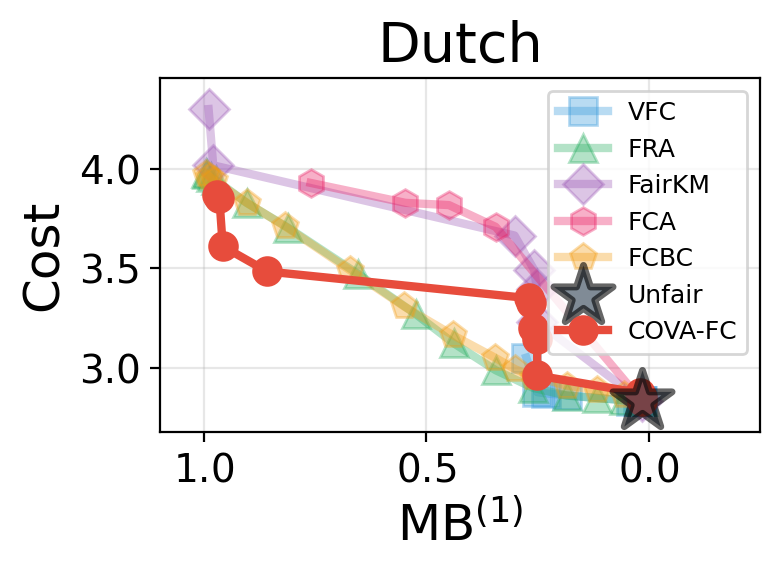}
    \includegraphics[width=0.3\linewidth]{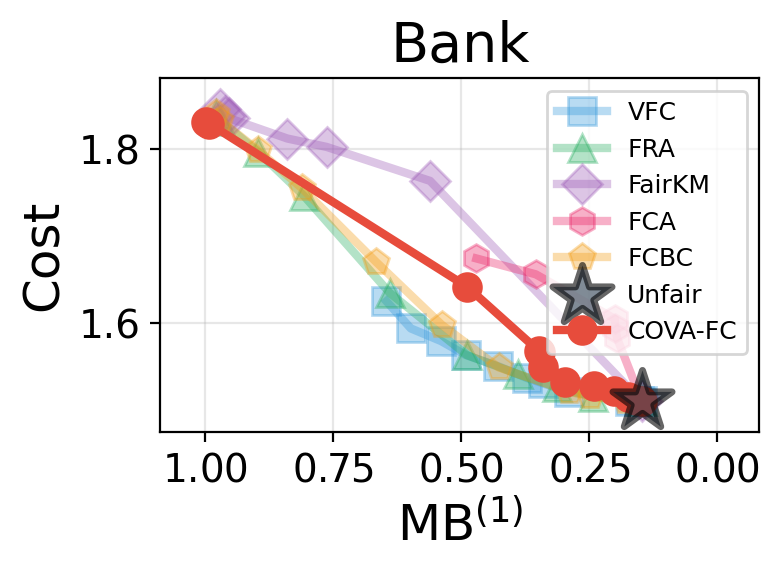}
    \caption{
    Comparison of trade-offs between subgroup Balance ($\textup{SB}$) and cost on (left to right) \texttt{Adult}, \texttt{Dutch}, and \texttt{Bank} datasets. 
    \texttt{Civilcomments} is excluded since its minimum subgroup size is smaller than $K$, which causes Balance to collapse to zero.
    }
    \label{fig:sb_trade_offs}
    \vskip -0.1in
\end{figure}

%%%%%%%%%%%%%%%
\newpage
%%%%%%%%%%%%%%%%%%%%%%%%%%%%%%%%%%%%%%%%%%%%%%%%%%%%%%%%%%%%%%%%%%%%%%%%
\subsection{Scalability.}\label{app:scalability}

To examine the scalability of \textsc{Cova-FC}, we evaluate runtime on a large-scale dataset and compare \textsc{Cova-FC} with VFC \citep{ziko2021variational}, the fastest baseline in \Cref{tab:runtime}.
We use the \texttt{ACSIncome} dataset with $n=1{,}664{,}500$ instances, feature dimension $d=7$, and $q=3$ sensitive attributes (gender, marital status, and race with categories White/Black/Asian/Other), which yields $16$ subgroups.
We run each method to reach a comparable fairness level.
The results in \Cref{tab:runtime_acs} show that \textsc{Cova-FC} runs $3.22\times$ faster than VFC at a comparable $\mathrm{SP}$ on this large-scale dataset.

\begin{table}[h!]
    \centering
    \footnotesize
    \caption{Comparison between \textsc{Cova-FC} and VFC on \texttt{ACSIncome} in terms of fairness level (SP) and runtime.
    }
    \begin{tabular}{lccc}
    \toprule
    Dataset & Method   & SP & Runtime (sec) \\
    \midrule
    
    \multirow{2}{*}{\texttt{ACSIncome}}
    & \textbf{\textsc{Cova-FC}$\checkmark$}  & 0.19 & 179.72  \\
    & VFC \citep{ziko2021variational}       & 0.18  & 578.75  \\
    
    \bottomrule
    \end{tabular}
    \label{tab:runtime_acs}
    \vskip -0.1in
\end{table}

%%%%%%%%%%%%%%%

% \newpage

%%%%%%%%%%%%%%%
\subsection{Relationship between $\overline{\textup{CR}}$ and $\Delta$}\label{sec:cr-gap}

\begin{figure*}[h!]
  \centering
  \includegraphics[width=0.23\linewidth]{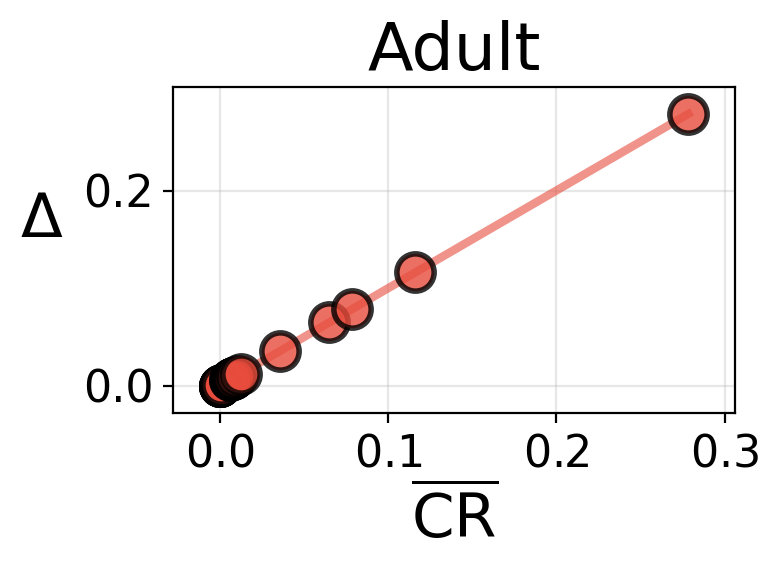}
  \includegraphics[width=0.23\linewidth]{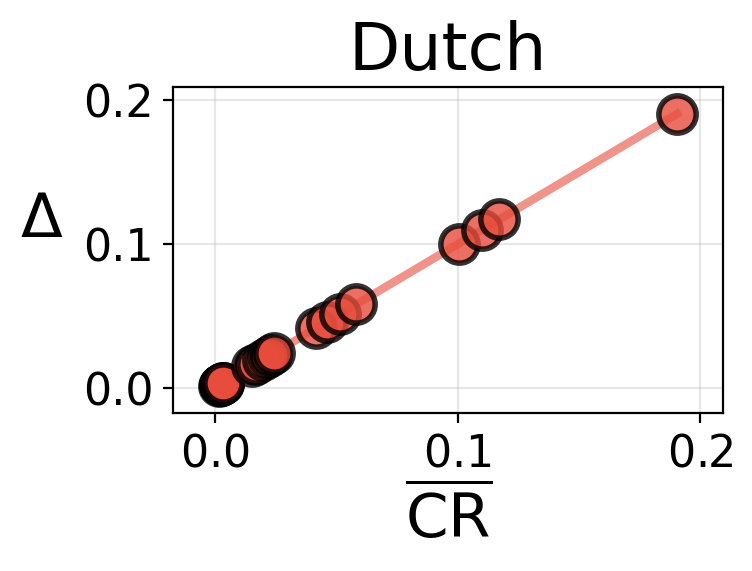}
  \includegraphics[width=0.23\linewidth]{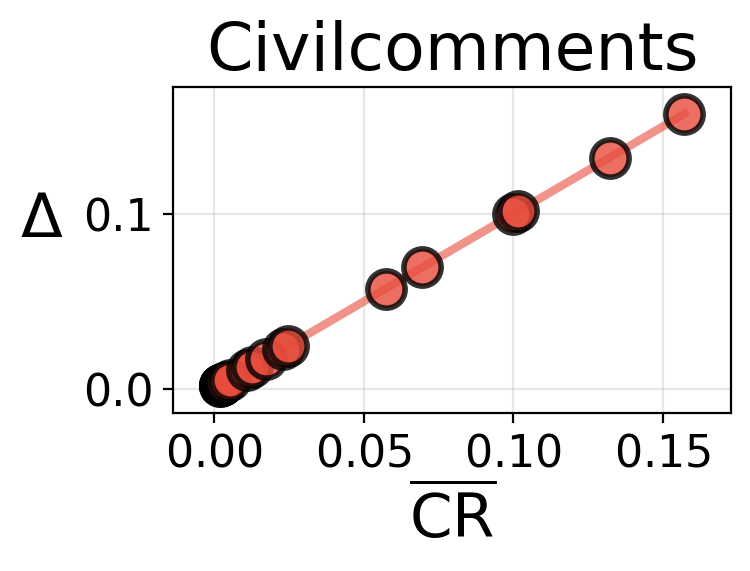}
  \includegraphics[width=0.23\linewidth]{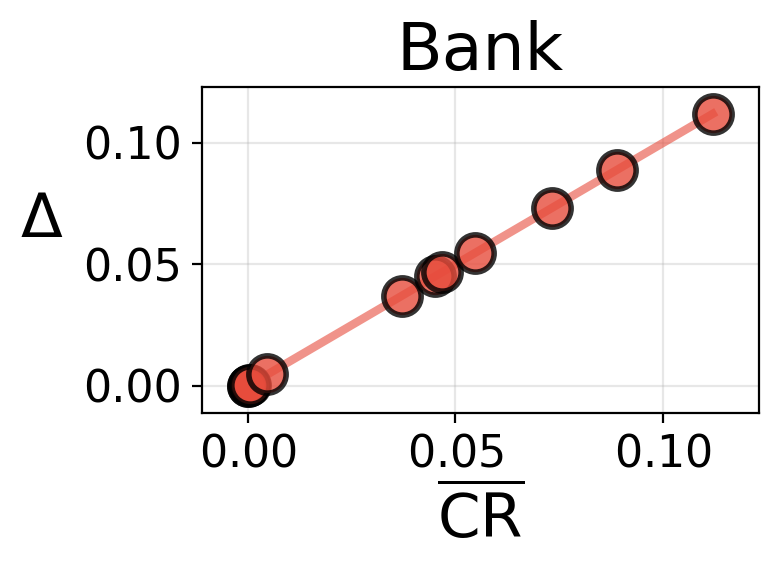}
  \caption{Plots of $\overline{\textup{CR}}$ vs. $\Delta$ on (left to right) \texttt{Adult}, \texttt{Dutch}, \texttt{Civilcomments}, and \texttt{Bank} datasets.
  }
  \label{fig:dr_gap_scatter_all}
  \vskip -0.1in
\end{figure*}

To empirically support \Cref{thm:Delta_ms_equiv_hard}, we visualize the relationship between $\overline{\textup{CR}}$ computed on the learned soft assignment and $\Delta$ computed on the induced hard assignment across different values of $\lambda.$
Although experiments are conducted with soft assignment, the learned assignment is highly concentrated near one-hot vectors in practice (see \Cref{sec:a_hard_converge}), so this comparison is informative for the final hard clustering as well.
\Cref{fig:dr_gap_scatter_all} shows a clear positive relation: smaller values of $\overline{\textup{CR}}$ are associated with smaller values of $\Delta,$ supporting $\overline{\textup{CR}}$ as an effective surrogate for optimizing fair hard clustering in practice.

%%%%%%%%%%%%%%%
\newpage 
\subsection{Deeper comparison with VFC \& FairKM}\label{app:vfc_deep_compare}

\paragraph{Why VFC is the closest baseline.}
Among the baselines, VFC is closest to \textsc{Cova-FC} in computational profile: both use soft assignment and per-instance parallel 
updates. We therefore compare the two methods more closely below.

Let \(\mathbf S^{(t)}\in[0,1]^{N\times K}\) denote the soft assignment matrix at iteration \(t\), and let \(S_k^{(t)}\) be its \(k\)-th column. Let \(V_j\in\{0,1\}^N\) be the indicator vector of subgroup \(j\), and let \(v_{j,p}\) be its \(p\)-th entry. Writing \(\mu_j\) for the target proportion of subgroup \(j\), \(a_{p,k}\) for the clustering term, and \(L\) for the bound-update scaling parameter, VFC assignment update is
\begin{equation}
S_{p,k}^{(t+1)}
\propto
S_{p,k}^{(t)}
\exp\!\left(
\frac{-a_{p,k}-\lambda b_{p,k}(\mathbf S^{(t)})}{L}
\right),
\end{equation}
followed by row normalization over \(k\), where
\begin{equation}
b_{p,k}^{(t)}
=
\sum_j
\left(
\frac{\mu_j}{\mathbf 1^\top S_k^{(t)}}
-
\frac{\mu_j v_{j,p}}{V_j^\top S_k^{(t)}}
\right).
\end{equation}

\paragraph{High-\(\lambda\) breakdown in VFC.}
The term \(V_j^\top S_k^{(t)}\) is the soft mass of subgroup \(j\) in cluster \(k\). 
When some subgroup-cluster masses are small, the inverse-mass term in \(b_{p,k}^{(t)}\) can become large. 
Since this term is multiplied by \(\lambda\) inside an exponential update, increasing \(\lambda\) can make the assignment update numerically extreme. 
In \Cref{tab:vfc_gradient_sp}, VFC initially improves SP on \texttt{Communities}. 
However, at slightly larger $\lambda$, the soft update becomes non-finite and the implementation falls back to the K-means solution. \textsc{Cova-FC} avoids this inverse subgroup-cluster mass term.

\begin{table}[h!]
\centering
\caption{VFC behavior on Communities with $K=10$, $L=2$, as $\lambda$ increases. 
}
\label{tab:vfc_gradient_sp}
\footnotesize
\setlength{\tabcolsep}{8pt}
\begin{tabular}{cccc}
\toprule
\(\lambda\) & Cost & SP & \(\max_{p,k}|\lambda b_{p,k}|\) \\
\midrule
\(0\) & \textbf{56.99} & \textbf{0.0251} & -- \\
\(1{\times}10^{-11}\) & 57.06 & 0.0240 & \(2.56{\times}10^2\) \\
\(2{\times}10^{-11}\) & 57.34 & 0.0216 & \(5.12{\times}10^2\) \\
\(4{\times}10^{-11}\) & 60.57 & 0.0141 & \(1.02{\times}10^3\) \\
\(6{\times}10^{-11}\) & \textbf{56.99} & \textbf{0.0251} & \(1.54{\times}10^3\) \\
\bottomrule
\end{tabular}
\end{table}

\paragraph{Why FairKM is hard to extend to soft assignment.}
FairKM \citep{abraham2020fairkm} updates cluster assignment by hard
reassignment, which can be viewed as solving
\[
\arg\min_{A_i \in [K]} 
\left(
\|x_i - \mu_{A_i}\|^2 
+ \lambda \cdot
\textsc{FairnessPenalty}(A_i; \text{others})
\right)
\]
for each instance \(i\) under the current assignment of the other instances.
The fairness penalty is evaluated from hard membership counts, so this
combinatorial update (i) requires evaluating a per-cluster fairness penalty for
each candidate \(A_i\), whose cost increases with both \(K\) and \(M\), where
\(M=2^q\); (ii) is difficult to parallelize across instances because each update
changes the global subgroup proportions; and (iii) is non-differentiable in
\(A_i\), ruling out direct gradient-based scaling.
A soft-assignment relaxation, replacing \(\mathbb I(A_i=k)\) with \(A_{ik}\),
does not preserve the original hard-count-based update structure of FairKM and
would require re-deriving the optimization.
In contrast, our continuous soft-assignment formulation directly optimizes a
differentiable subgroup-fairness surrogate, allowing fine-grained control over
small subgroups within a scalable optimization framework.

%%%%%%%%%%%%%%%
\newpage

\subsection{Effects of the continuous relaxation}\label{app:effect_relax}

\paragraph{Four variants of the surrogate.}
Recall that $\textup{CR}(\mathbf A)$ has two discrete components: the worst-case subgroup membership $s$ (relaxed by $\mathbf v \in \mathbb S^{M_+}$) and the hard cluster assignment $\mathbf A$ (relaxed by replacing $\mathbb I(A_i = k)$ with $A_{ik}$).
Depending on which component is relaxed, we have four variants:
\begin{align*}
    \textup{CR}(\mathbf A)
    &:= \max_{s \in \mathcal S_+} \max_{\beta \in \mathcal B_\infty}
       \tfrac{1}{2n} \textstyle\sum_{i=1}^n \left(\sum_{k=1}^K \beta_k \mathbb I(A_i = k)\right)(c_{is} - \bar c_s), \\
    \textup{CR}_{\mathbf v}(\mathbf A)
    &:= \max_{\mathbf v \in \mathbb S^{M_+}} \max_{\beta \in \mathcal B_\infty}
       \tfrac{1}{2n} \textstyle\sum_{i=1}^n \left(\sum_{k=1}^K \beta_k \mathbb I(A_i = k)\right)
       \mathbf v^\top (c_i - \bar c), \\
    \textup{CR}_{\mathbf A}(\mathbf A)
    &:= \max_{s \in \mathcal S_+} \max_{\beta \in \mathcal B_\infty}
       \tfrac{1}{2n} \textstyle\sum_{i=1}^n \left(\sum_{k=1}^K \beta_k A_{ik}\right)(c_{is} - \bar c_s), \\
    \overline{\textup{CR}}(\mathbf A)
    &:= \max_{\mathbf v \in \mathbb S^{M_+}} \max_{\beta \in \mathcal B_\infty}
       \tfrac{1}{2n} \textstyle\sum_{i=1}^n \left(\sum_{k=1}^K \beta_k A_{ik}\right)
       \mathbf v^\top (c_i - \bar c).
\end{align*}
Here $\textup{CR}$ is the original (both hard), $\textup{CR}_{\mathbf v}$ relaxes only $\mathbf v$, $\textup{CR}_{\mathbf A}$ relaxes only $\mathbf A$, and $\overline{\textup{CR}}$ relaxes both. The same naming applies to the general $\mathcal W$-version $\textup{CR}(\mathbf A; \mathcal W)$ by replacing the max over $s$ with a max over $W \in \mathcal W_+$ and using the corresponding $c_{i,W}, \bar c_W$.
$\overline{\textup{CR}}$ is the final surrogate used in \textsc{Cova-FC}; the other three serve as ablations below.

%%%%%%%%%%%%%%%
\paragraph{$\textup{CR}_{\mathbf v}$ vs.\ $\textup{CR}_{\mathbf A}$.}
We compare $\textup{CR}_{\mathbf v}$ (hard $\mathbf A$) with $\textup{CR}_{\mathbf A}$ (soft $\mathbf A$) on the four benchmark datasets.
\Cref{tab:abl_A_all_tilde} shows that $\textup{CR}_{\mathbf A}$ produces fair clusters, whereas $\textup{CR}_{\mathbf v}$ often fails to reduce the fairness gap due to the non-smooth indicator $\mathbb I(A_i=k)$.

\begin{table}[hbt!]
\centering
\footnotesize
\caption{Comparison between $\textup{CR}_{\mathbf v}$ and $\textup{CR}_{\mathbf A}$ in terms of Cost and SP.}
\label{tab:abl_A_all_tilde}
\begin{tabular}{lccc}
\toprule
Dataset & Variant & Cost & SP \\
\midrule
\multirow{2}{*}{\texttt{Adult}}
& $\textup{CR}_{\mathbf v}$ & 1.70 & 0.1046 \\
& $\textup{CR}_{\mathbf A}$ & 2.52 & \textbf{0.0002} \\
\midrule

\multirow{2}{*}{\texttt{Dutch}}
& $\textup{CR}_{\mathbf v}$ & 2.85 & 0.0888 \\
& $\textup{CR}_{\mathbf A}$ & 3.94 & \textbf{0.0001} \\
\midrule

\multirow{2}{*}{\texttt{Civilcomments}}
& $\textup{CR}_{\mathbf v}$ & 624.99 & 0.0711 \\
& $\textup{CR}_{\mathbf A}$ & 657.84 & \textbf{0.0012} \\
\midrule

\multirow{2}{*}{\texttt{Bank}}
& $\textup{CR}_{\mathbf v}$ & 1.52 & 0.0877 \\
& $\textup{CR}_{\mathbf A}$ & 1.83 & \textbf{0.0001} \\
\bottomrule
\end{tabular}
\vskip -0.1in
\end{table}

%%%%%%%%%%%%%%%
\paragraph{$\textup{CR}_{\mathbf A}$ vs.\ $\overline{\textup{CR}}$.}
We compare $\textup{CR}_{\mathbf A}$ (hard $\mathbf v$, requiring an explicit max over all subgroups) with $\overline{\textup{CR}}$ (soft $\mathbf v$, optimizing a linear combination via $\mathbf v$).
$\overline{\textup{CR}}$ avoids the linear scan and admits closed-form gradients in $\mathbf v$.
\Cref{tab:abl_A_all} shows that $\overline{\textup{CR}}$ matches the cost and SP of $\textup{CR}_{\mathbf A}$ at substantially lower runtime, with larger gains when $|\mathcal W|$ is larger (\texttt{Adult}: 16 subgroups, \texttt{Civilcomments}: 24).

\begin{table}[hbt!]
\centering
\footnotesize
\caption{Comparison between $\textup{CR}_{\mathbf A}$ and $\overline{\textup{CR}}$ in terms of Cost, SP, and Runtime.}
\begin{tabular}{lcccc}
\toprule
Dataset (\# Subgroups) & Variant & Cost & SP & Runtime (sec) \\
\midrule

\multirow{2}{*}{\texttt{Adult} (16)}
& $\textup{CR}_{\mathbf A}$ & 2.53 & 0.0002 & 201.66 \\
& $\overline{\textup{CR}}$ & 2.52 & 0.0002 & \textbf{132.80} ($\times 1.52$ faster) \\
\midrule

\multirow{2}{*}{\texttt{Dutch} (4)}
& $\textup{CR}_{\mathbf A}$ & 3.94 & 0.0001 & 34.60 \\
& $\overline{\textup{CR}}$ & 3.94 & 0.0001 & \textbf{32.00} ($\times 1.08$ faster) \\
\midrule

\multirow{2}{*}{\texttt{Civilcomments} (24)}
& $\textup{CR}_{\mathbf A}$ & 657.84 & 0.0012 & 224.80 \\
& $\overline{\textup{CR}}$ & 657.84 & 0.0012 & \textbf{89.40} ($\times 2.51$ faster) \\
\midrule

\multirow{2}{*}{\texttt{Bank} (3)}
& $\textup{CR}_{\mathbf A}$ & 1.83 & 0.0001 & 42.12 \\
& $\overline{\textup{CR}}$ & 1.83 & 0.0001 & \textbf{39.43} ($\times 1.07$ faster) \\

\bottomrule
\end{tabular}
\label{tab:abl_A_all}
\vskip -0.1in
\end{table}

%%%%%%
\newpage
\paragraph{$\textup{CR}_{\mathbf v}$ vs.\ $\overline{\textup{CR}}$.}
We further isolate the role of relaxing $\mathbf A$ by comparing $\overline{\textup{CR}}$ with $\textup{CR}_{\mathbf v}$, which keeps $\mathbf v$ relaxed but uses the hard indicator on $\mathbf A$.
With $(\mathbf \mu, \mathbf v, \beta)$ fixed, $\textup{CR}_{\mathbf v}$ admits the closed-form update
$
A_i \leftarrow \arg\min_{k \in [K]}\left(\|x_i - \mu_k\|^2 + \lambda \beta_k \mathbf v^\top (c_i - \bar c)\right),\ i \in [n].
$

\Cref{fig:abl_soft_v_hard_A} shows that $\overline{\textup{CR}}$ yields a stable decrease and clear convergence, whereas $\textup{CR}_{\mathbf v}$ exhibits oscillations and fails to converge, because the non-smooth indicator $\mathbb I(A_i = k)$ makes the objective discontinuous in $\mathbf A$ and induces frequent switching of the maximizers $(\mathbf v, \beta)$.

\begin{figure}[h!]
  \centering
  \includegraphics[width=0.3\linewidth]{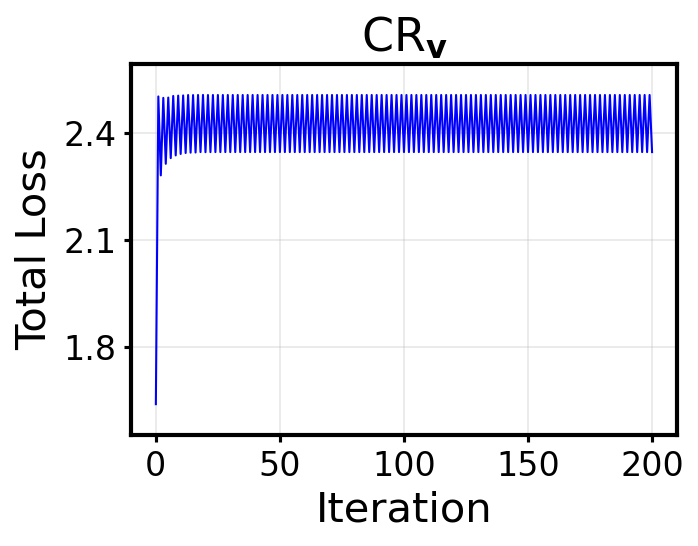}
  \includegraphics[width=0.3\linewidth]{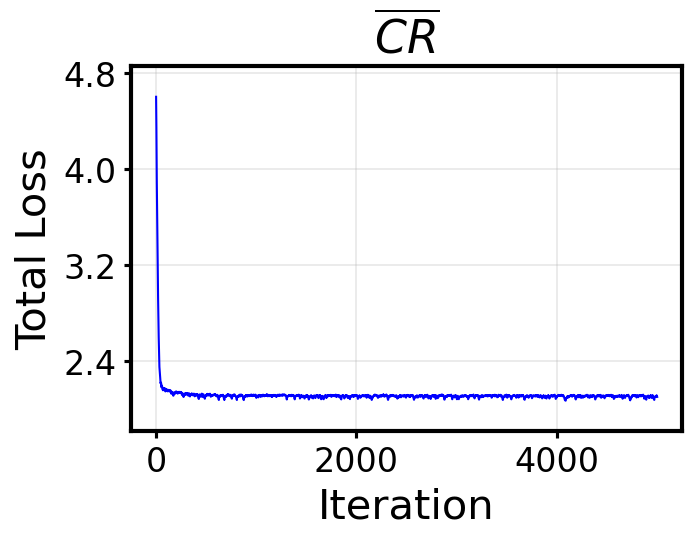}
    \caption{Plots of the total loss (clustering cost $+\lambda$ fairness penalty) across epochs on \texttt{Adult}, using $\textup{CR}_{\mathbf v}$ (left) and $\overline{\textup{CR}}$ (right) as the fairness penalty.}
  \label{fig:abl_soft_v_hard_A}
  \vskip -0.1in
\end{figure}

In addition to the total loss, we also track the fairness gap  $\Delta(\mathbf A; \mathcal W)$ over iterations.
\Cref{fig:app_abl_soft_v_hard_A} presents the results across multiple values of $\lambda:$
(top) $\textup{CR}_{\mathbf v}$ with hard cluster assignment as the fairness penalty.
(bottom) $\overline{\textup{CR}}$ with soft cluster assignment as the fairness penalty.
The number of iterations can differ across settings because each run is terminated once the objective in \Cref{eq:obj} satisfies the convergence criterion (see \Cref{sec:exp-setting}).
We observe that $\textup{CR}_{\mathbf v}$ (top) (i) exhibits oscillations, (ii) attains substantially larger fairness gap than $\overline{\textup{CR}},$ and (iii) does not show a clear improvement in the gap even when using a large $\lambda.$
In contrast, $\overline{\textup{CR}}$ (bottom) reduces the gap clearly and produces fair clusters.

\begin{figure}[h!]
  \centering
  \includegraphics[width=0.24\linewidth]{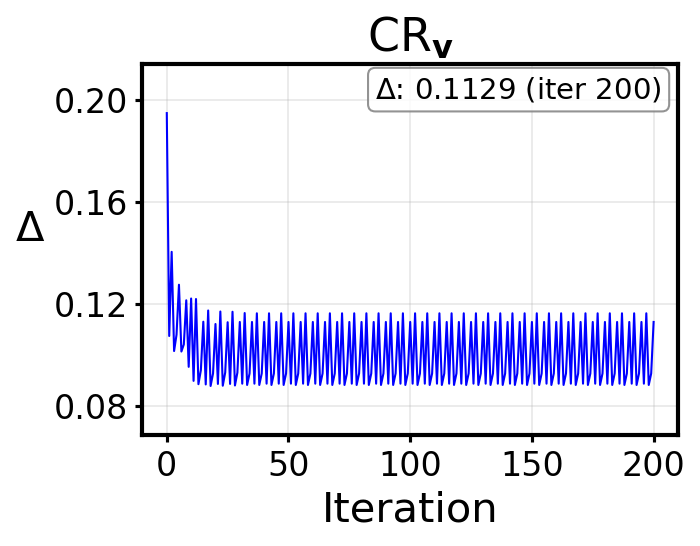}
  \includegraphics[width=0.24\linewidth]{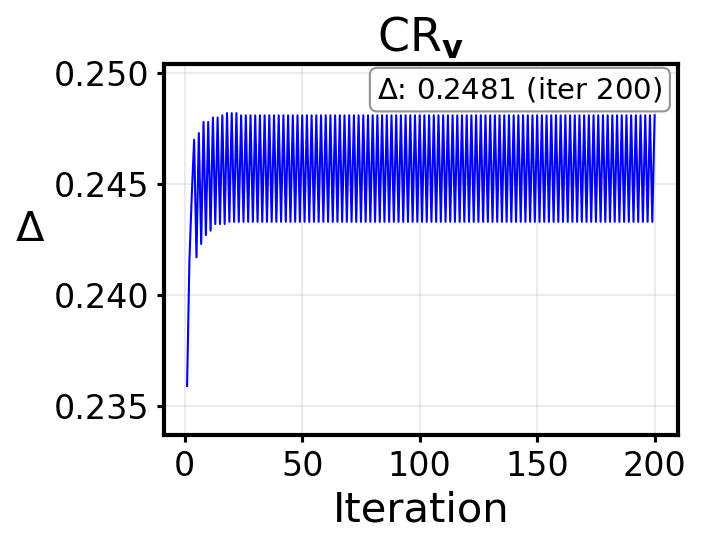}
  \includegraphics[width=0.24\linewidth]{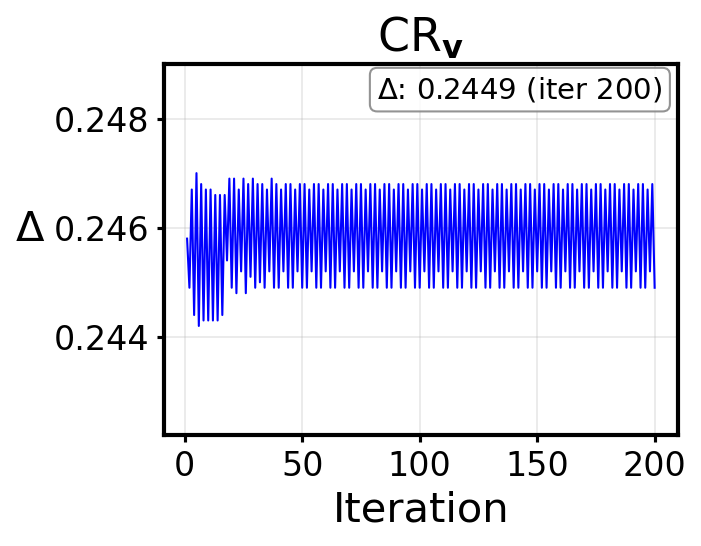}
  \includegraphics[width=0.24\linewidth]{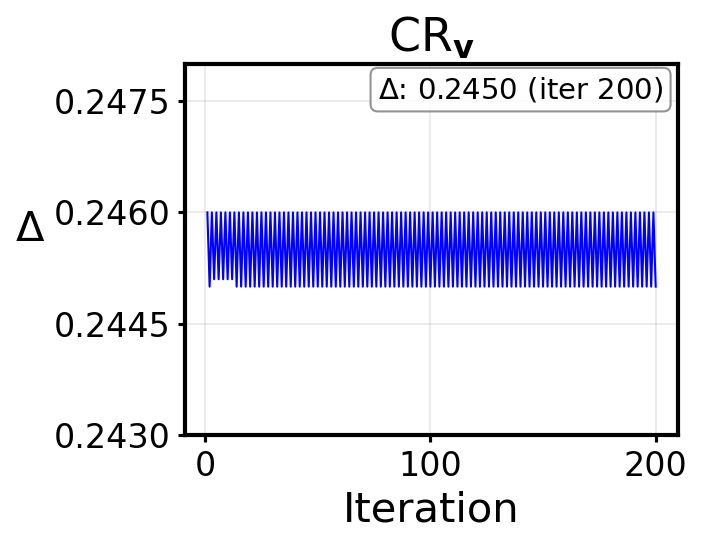}
  
  \includegraphics[width=0.24\linewidth]{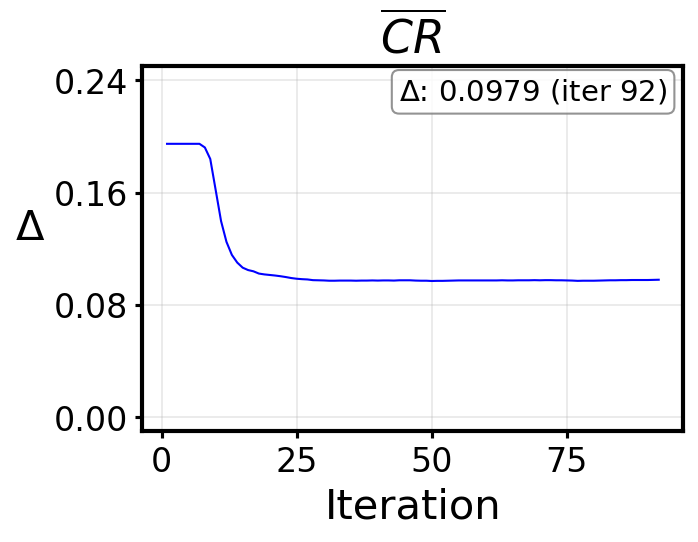}
  \includegraphics[width=0.24\linewidth]{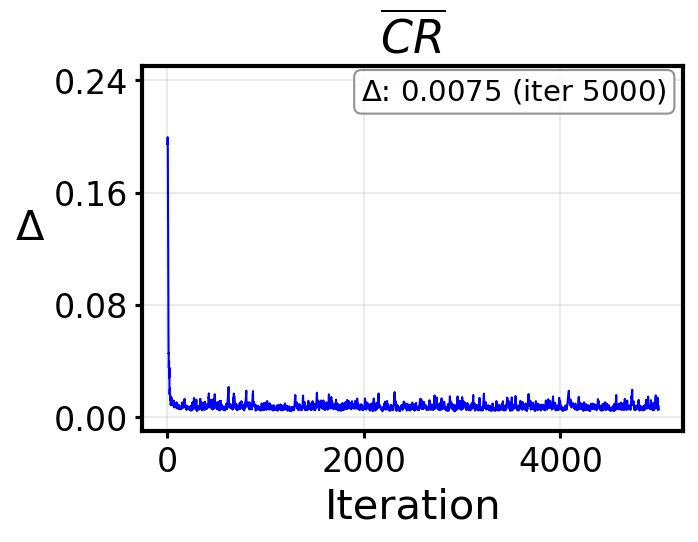} 
  \includegraphics[width=0.24\linewidth]{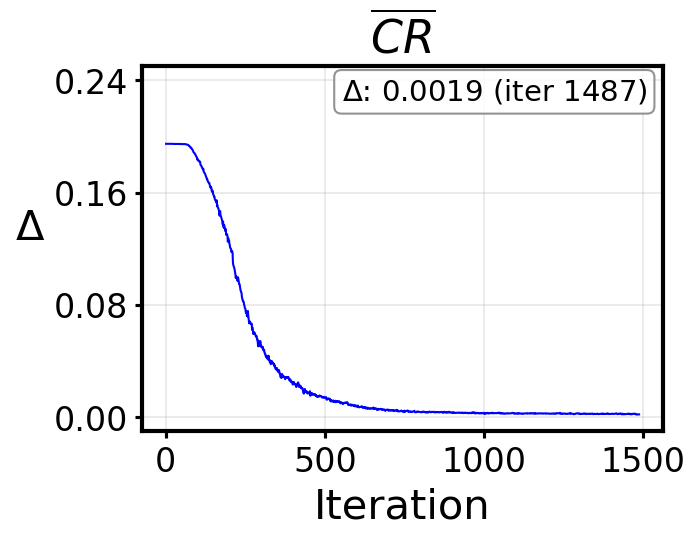}
  \includegraphics[width=0.24\linewidth]{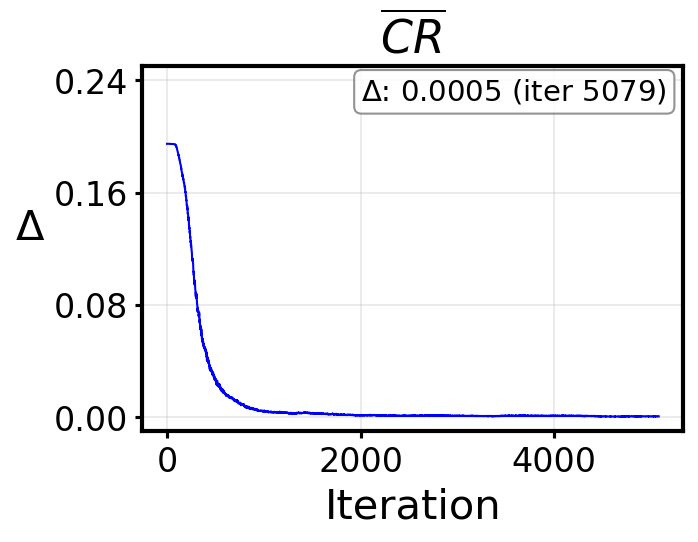}
  
  \caption{Plots of $\Delta(\mathbf A; \mathcal W)$ across iterations on \texttt{Adult} dataset when using $\textup{CR}_{\mathbf v}$ (top) and $\overline{\textup{CR}}$ (bottom) as the fairness penalty.
  (left to right) $\lambda$ ranges from $0.3$ to $50$.}
  % $\lambda = 0.02, 7.00, 15.00, 30.00.$}
  \label{fig:app_abl_soft_v_hard_A}
  \vskip -0.1in
\end{figure}

Overall, these results indicate that using \textit{hard cluster assignment harms optimization stability even when the most-violated selection is relaxed} by the use of $\mathbf v$ instead of $m,$ and often fails to produce fair clusters in practice.

%%%%%%
\newpage
\paragraph{Concentration of soft cluster assignment.}\label{sec:a_hard_converge}

We examine how concentrated each soft cluster assignment $A_{i}$ becomes during optimization.
Specifically, for each $i\in[n]$ we compute the maximum probability $\max_{k\in[K]} A_{ik}$, and plot histograms of these values across all instances at 100, 300, 500, and 700 iterations (\Cref{fig:assignment_max_hist}).
As optimization proceeds, the distribution of $\max_{k\in[K]} A_{ik}$ increasingly concentrates near $1$, indicating that each $A_i$ becomes sharply peaked and thus close to a hard cluster assignment.
This result justifies the use of soft assignment even if our goal is to find a fair hard assignment.

\begin{figure*}[h!]
  \centering
  \includegraphics[width=0.24\linewidth]{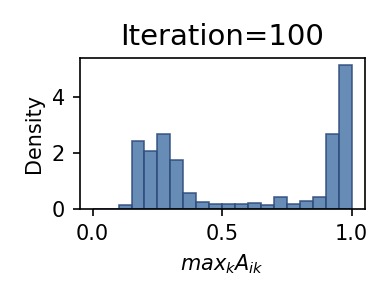}
  \includegraphics[width=0.24\linewidth]{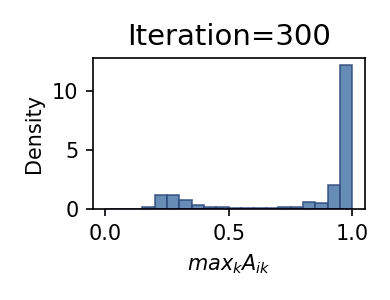}
  \includegraphics[width=0.24\linewidth]{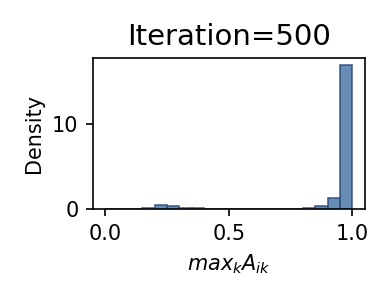}
  \includegraphics[width=0.24\linewidth]{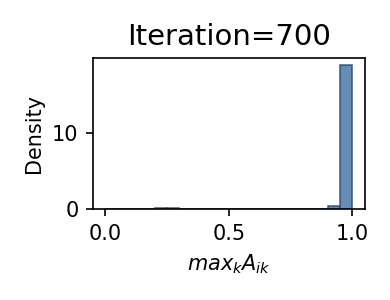}
  \caption{
    Histograms of the maximum soft cluster assignment, $\max_{k\in[K]} A_{ik}$, over all instances $i\in[n]$ at 100, 300, 500, and 700 iterations.}
  \label{fig:assignment_max_hist}
  \vskip -0.1in
\end{figure*}

%%%%%%%%%
% \newpage
\subsection{Additional image experiment on CelebA}

We additionally evaluate \textsc{Cova-FC} on the \texttt{CelebA} dataset to test its applicability in an image-feature setting. 
We use four binary sensitive attributes, which induce 16 subgroups, and extract feature vectors using a pre-trained ResNet-18 model. 
We then apply \textsc{Cova-FC} to these extracted features in the same way as for the other datasets.
The results in \Cref{fig:celeba_tradeoff} show that \textsc{Cova-FC} can be applied without modification to image-derived feature representations and still improves subgroup fairness as the fairness penalty increases. 

\begin{figure*}[h!]
  \centering
  \includegraphics[width=0.3\linewidth]{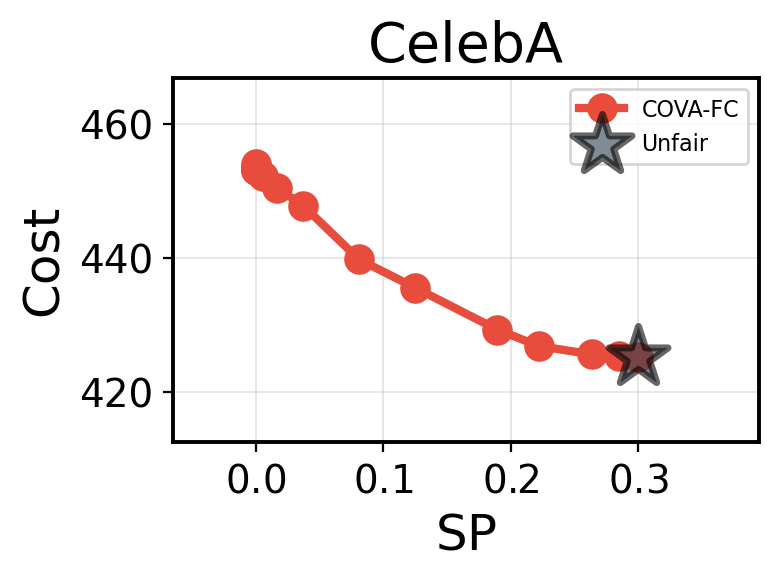}
  \includegraphics[width=0.3\linewidth]{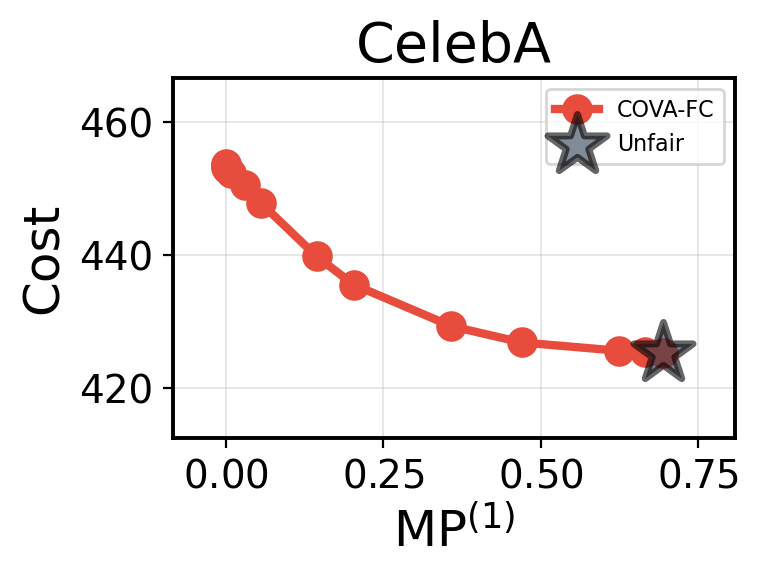}

  \caption{Cost-fairness trade-off on the \texttt{CelebA} dataset with four binary sensitive attributes (16 subgroups), using image features extracted by a pre-trained ResNet-18 model.
  }
  \label{fig:celeba_tradeoff}
  % \vskip -0.1in
\end{figure*}

\section{Broader impact}\label{app:broader_impact}

\textsc{Cova-FC} targets subgroup unfairness under multiple sensitive attributes, a setting common in real applications such as customer segmentation, resource allocation, and large-scale demographic analysis, where unfair treatment is often driven by intersections of attributes rather than a single one.
By controlling subgroup and marginal fairness simultaneously and admitting parallel updates that scale to millions of instances, \textsc{Cova-FC} broadens the deployability of fair clustering in practice.
At the same time, \textsc{Cova-FC} enforces only a statistical fairness criterion at the clustering level: a balanced cluster assignment does not by itself guarantee fair outcomes in downstream tasks (e.g., recommendation, screening, or pricing), and the choice of sensitive attributes inevitably reflects pre-existing social categories with their own simplifications.
Practitioners should therefore audit fairness in the deployment context, document the chosen attributes and groupings, and avoid relying on a single fairness measure as a guarantee of equitable treatment.

\newpage
\section*{NeurIPS Paper Checklist}

\begin{enumerate}

\item {\bf Claims}
    \item[] Question: Do the main claims made in the abstract and introduction accurately reflect the paper's contributions and scope?
    \item[] Answer: \answerYes{}
    \item[] Justification: The abstract and \Cref{sec:exp} state our contributions: a covariance-based fairness gap, its continuous relaxation, the resulting algorithm \textsc{Cova-FC}, an extension to subgroup-marginal fairness, and empirical evaluation against existing baselines.
    \item[] Guidelines:
    \begin{itemize}
        \item The answer \answerNA{} means that the abstract and introduction do not include the claims made in the paper.
        \item The abstract and/or introduction should clearly state the claims made, including the contributions made in the paper and important assumptions and limitations. A \answerNo{} or \answerNA{} answer to this question will not be perceived well by the reviewers. 
        \item The claims made should match theoretical and experimental results, and reflect how much the results can be expected to generalize to other settings. 
        \item It is fine to include aspirational goals as motivation as long as it is clear that these goals are not attained by the paper. 
    \end{itemize}

\item {\bf Limitations}
    \item[] Question: Does the paper discuss the limitations of the work performed by the authors?
    \item[] Answer: \answerYes{}
    \item[] Justification: We discuss limitations and future directions in the concluding remarks (\Cref{sec:cr-comparison} and the Future directions paragraph), including the binary-attribute assumption (extended in a Remark in \Cref{sec:exp-setting}) and possible extensions to other clustering paradigms.
    \item[] Guidelines:
    \begin{itemize}
        \item The answer \answerNA{} means that the paper has no limitation while the answer \answerNo{} means that the paper has limitations, but those are not discussed in the paper. 
        \item The authors are encouraged to create a separate ``Limitations'' section in their paper.
        \item The paper should point out any strong assumptions and how robust the results are to violations of these assumptions (e.g., independence assumptions, noiseless settings, model well-specification, asymptotic approximations only holding locally). The authors should reflect on how these assumptions might be violated in practice and what the implications would be.
        \item The authors should reflect on the scope of the claims made, e.g., if the approach was only tested on a few datasets or with a few runs. In general, empirical results often depend on implicit assumptions, which should be articulated.
        \item The authors should reflect on the factors that influence the performance of the approach. For example, a facial recognition algorithm may perform poorly when image resolution is low or images are taken in low lighting. Or a speech-to-text system might not be used reliably to provide closed captions for online lectures because it fails to handle technical jargon.
        \item The authors should discuss the computational efficiency of the proposed algorithms and how they scale with dataset size.
        \item If applicable, the authors should discuss possible limitations of their approach to address problems of privacy and fairness.
        \item While the authors might fear that complete honesty about limitations might be used by reviewers as grounds for rejection, a worse outcome might be that reviewers discover limitations that aren't acknowledged in the paper. The authors should use their best judgment and recognize that individual actions in favor of transparency play an important role in developing norms that preserve the integrity of the community. Reviewers will be specifically instructed to not penalize honesty concerning limitations.
    \end{itemize}

\item {\bf Theory assumptions and proofs}
    \item[] Question: For each theoretical result, does the paper provide the full set of assumptions and a complete (and correct) proof?
    \item[] Answer: \answerYes{}
    \item[] Justification: All theorems (\Cref{thm:Delta_CR_equal,thm:Delta_CRbar_equal,thm:Delta_ms_equiv_hard}) are stated with explicit assumptions in the main paper, and full proofs are provided in \Cref{sec:appen-proofs}.
    \item[] Guidelines:
    \begin{itemize}
        \item The answer \answerNA{} means that the paper does not include theoretical results. 
        \item All the theorems, formulas, and proofs in the paper should be numbered and cross-referenced.
        \item All assumptions should be clearly stated or referenced in the statement of any theorems.
        \item The proofs can either appear in the main paper or the supplemental material, but if they appear in the supplemental material, the authors are encouraged to provide a short proof sketch to provide intuition. 
        \item Inversely, any informal proof provided in the core of the paper should be complemented by formal proofs provided in appendix or supplemental material.
        \item Theorems and Lemmas that the proof relies upon should be properly referenced. 
    \end{itemize}

    \item {\bf Experimental result reproducibility}
    \item[] Question: Does the paper fully disclose all the information needed to reproduce the main experimental results of the paper to the extent that it affects the main claims and/or conclusions of the paper (regardless of whether the code and data are provided or not)?
    \item[] Answer: \answerYes{}
    \item[] Justification: Datasets, baselines, sensitive attributes, hyperparameters, and convergence criteria are described in \Cref{sec:exp-setting} and \Cref{sec:appen-impl}; the algorithm pseudocode and update derivations are in \Cref{app:cova_fc_algorithm}.
    \item[] Guidelines:
    \begin{itemize}
        \item The answer \answerNA{} means that the paper does not include experiments.
        \item If the paper includes experiments, a \answerNo{} answer to this question will not be perceived well by the reviewers: Making the paper reproducible is important, regardless of whether the code and data are provided or not.
        \item If the contribution is a dataset and\slash or model, the authors should describe the steps taken to make their results reproducible or verifiable. 
        \item Depending on the contribution, reproducibility can be accomplished in various ways. For example, if the contribution is a novel architecture, describing the architecture fully might suffice, or if the contribution is a specific model and empirical evaluation, it may be necessary to either make it possible for others to replicate the model with the same dataset, or provide access to the model. In general. releasing code and data is often one good way to accomplish this, but reproducibility can also be provided via detailed instructions for how to replicate the results, access to a hosted model (e.g., in the case of a large language model), releasing of a model checkpoint, or other means that are appropriate to the research performed.
        \item While NeurIPS does not require releasing code, the conference does require all submissions to provide some reasonable avenue for reproducibility, which may depend on the nature of the contribution. For example
        \begin{enumerate}
            \item If the contribution is primarily a new algorithm, the paper should make it clear how to reproduce that algorithm.
            \item If the contribution is primarily a new model architecture, the paper should describe the architecture clearly and fully.
            \item If the contribution is a new model (e.g., a large language model), then there should either be a way to access this model for reproducing the results or a way to reproduce the model (e.g., with an open-source dataset or instructions for how to construct the dataset).
            \item We recognize that reproducibility may be tricky in some cases, in which case authors are welcome to describe the particular way they provide for reproducibility. In the case of closed-source models, it may be that access to the model is limited in some way (e.g., to registered users), but it should be possible for other researchers to have some path to reproducing or verifying the results.
        \end{enumerate}
    \end{itemize}

\item {\bf Open access to data and code}
    \item[] Question: Does the paper provide open access to the data and code, with sufficient instructions to faithfully reproduce the main experimental results, as described in supplemental material?
    \item[] Answer: \answerYes{}
    \item[] Justification: All datasets used are publicly available and properly cited (\Cref{sec:appen-impl}); we will release anonymized code with the camera-ready version, and detailed implementation specifications are provided in \Cref{app:cova_fc_algorithm,sec:appen-impl}.
    \item[] Guidelines:
    \begin{itemize}
        \item The answer \answerNA{} means that paper does not include experiments requiring code.
        \item Please see the NeurIPS code and data submission guidelines (\url{https://neurips.cc/public/guides/CodeSubmissionPolicy}) for more details.
        \item While we encourage the release of code and data, we understand that this might not be possible, so \answerNo{} is an acceptable answer. Papers cannot be rejected simply for not including code, unless this is central to the contribution (e.g., for a new open-source benchmark).
        \item The instructions should contain the exact command and environment needed to run to reproduce the results. See the NeurIPS code and data submission guidelines (\url{https://neurips.cc/public/guides/CodeSubmissionPolicy}) for more details.
        \item The authors should provide instructions on data access and preparation, including how to access the raw data, preprocessed data, intermediate data, and generated data, etc.
        \item The authors should provide scripts to reproduce all experimental results for the new proposed method and baselines. If only a subset of experiments are reproducible, they should state which ones are omitted from the script and why.
        \item At submission time, to preserve anonymity, the authors should release anonymized versions (if applicable).
        \item Providing as much information as possible in supplemental material (appended to the paper) is recommended, but including URLs to data and code is permitted.
    \end{itemize}

\item {\bf Experimental setting/details}
    \item[] Question: Does the paper specify all the training and test details (e.g., data splits, hyperparameters, how they were chosen, type of optimizer) necessary to understand the results?
    \item[] Answer: \answerYes{}
    \item[] Justification: Optimizer choice, learning rate schedule, the fairness penalty $\lambda$, initialization, and convergence criteria are specified in \Cref{sec:exp-setting} and \Cref{sec:appen-impl}.
    \item[] Guidelines:
    \begin{itemize}
        \item The answer \answerNA{} means that the paper does not include experiments.
        \item The experimental setting should be presented in the core of the paper to a level of detail that is necessary to appreciate the results and make sense of them.
        \item The full details can be provided either with the code, in appendix, or as supplemental material.
    \end{itemize}

\item {\bf Experiment statistical significance}
    \item[] Question: Does the paper report error bars suitably and correctly defined or other appropriate information about the statistical significance of the experiments?
    \item[] Answer: \answerYes{}
    \item[] Justification: In \Cref{sec:large-number-subgroup}, we report results averaged over five random seeds. For the remaining experiments, we follow a shared protocol across baselines and report cost-fairness trade-off curves across multiple penalty levels and datasets; the observed trends are consistent across \Cref{sec:exp,sec:app-exp}.
    \item[] Guidelines:
    \begin{itemize}
        \item The answer \answerNA{} means that the paper does not include experiments.
        \item The authors should answer \answerYes{} if the results are accompanied by error bars, confidence intervals, or statistical significance tests, at least for the experiments that support the main claims of the paper.
        \item The factors of variability that the error bars are capturing should be clearly stated (for example, train/test split, initialization, random drawing of some parameter, or overall run with given experimental conditions).
        \item The method for calculating the error bars should be explained (closed form formula, call to a library function, bootstrap, etc.)
        \item The assumptions made should be given (e.g., Normally distributed errors).
        \item It should be clear whether the error bar is the standard deviation or the standard error of the mean.
        \item It is OK to report 1-sigma error bars, but one should state it. The authors should preferably report a 2-sigma error bar than state that they have a 96\% CI, if the hypothesis of Normality of errors is not verified.
        \item For asymmetric distributions, the authors should be careful not to show in tables or figures symmetric error bars that would yield results that are out of range (e.g., negative error rates).
        \item If error bars are reported in tables or plots, the authors should explain in the text how they were calculated and reference the corresponding figures or tables in the text.
    \end{itemize}

\item {\bf Experiments compute resources}
    \item[] Question: For each experiment, does the paper provide sufficient information on the computer resources (type of compute workers, memory, time of execution) needed to reproduce the experiments?
    \item[] Answer: \answerYes{}
    \item[] Justification: The hardware setup and per-experiment runtimes are reported in \Cref{sec:appen-impl} and the runtime tables (\Cref{tab:runtime}).
    \item[] Guidelines:
    \begin{itemize}
        \item The answer \answerNA{} means that the paper does not include experiments.
        \item The paper should indicate the type of compute workers CPU or GPU, internal cluster, or cloud provider, including relevant memory and storage.
        \item The paper should provide the amount of compute required for each of the individual experimental runs as well as estimate the total compute. 
        \item The paper should disclose whether the full research project required more compute than the experiments reported in the paper (e.g., preliminary or failed experiments that didn't make it into the paper). 
    \end{itemize}
    
\item {\bf Code of ethics}
    \item[] Question: Does the research conducted in the paper conform, in every respect, with the NeurIPS Code of Ethics \url{https://neurips.cc/public/EthicsGuidelines}?
    \item[] Answer: \answerYes{}
    \item[] Justification: We have reviewed the NeurIPS Code of Ethics, and the research conforms in all respects, including anonymity in submission and the use of publicly available datasets.
    \item[] Guidelines:
    \begin{itemize}
        \item The answer \answerNA{} means that the authors have not reviewed the NeurIPS Code of Ethics.
        \item If the authors answer \answerNo, they should explain the special circumstances that require a deviation from the Code of Ethics.
        \item The authors should make sure to preserve anonymity (e.g., if there is a special consideration due to laws or regulations in their jurisdiction).
    \end{itemize}

\item {\bf Broader impacts}
    \item[] Question: Does the paper discuss both potential positive societal impacts and negative societal impacts of the work performed?
    \item[] Answer: \answerYes{}
    \item[] Justification: \Cref{app:broader_impact} discusses positive impacts (mitigating subgroup unfairness in real applications, with scalability that broadens deployability), 
    negative impacts (only a statistical criterion at the clustering level, no guarantee of downstream fairness), and practitioner guidelines.
    \item[] Guidelines:
    \begin{itemize}
        \item The answer \answerNA{} means that there is no societal impact of the work performed.
        \item If the authors answer \answerNA{} or \answerNo, they should explain why their work has no societal impact or why the paper does not address societal impact.
        \item Examples of negative societal impacts include potential malicious or unintended uses (e.g., disinformation, generating fake profiles, surveillance), fairness considerations (e.g., deployment of technologies that could make decisions that unfairly impact specific groups), privacy considerations, and security considerations.
        \item The conference expects that many papers will be foundational research and not tied to particular applications, let alone deployments. However, if there is a direct path to any negative applications, the authors should point it out. For example, it is legitimate to point out that an improvement in the quality of generative models could be used to generate Deepfakes for disinformation. On the other hand, it is not needed to point out that a generic algorithm for optimizing neural networks could enable people to train models that generate Deepfakes faster.
        \item The authors should consider possible harms that could arise when the technology is being used as intended and functioning correctly, harms that could arise when the technology is being used as intended but gives incorrect results, and harms following from (intentional or unintentional) misuse of the technology.
        \item If there are negative societal impacts, the authors could also discuss possible mitigation strategies (e.g., gated release of models, providing defenses in addition to attacks, mechanisms for monitoring misuse, mechanisms to monitor how a system learns from feedback over time, improving the efficiency and accessibility of ML).
    \end{itemize}
    
\item {\bf Safeguards}
    \item[] Question: Does the paper describe safeguards that have been put in place for responsible release of data or models that have a high risk for misuse (e.g., pre-trained language models, image generators, or scraped datasets)?
    \item[] Answer: \answerNA{}
    \item[] Justification: The paper does not release any pre-trained generative model or scraped dataset; the contribution is a fair clustering algorithm with no high risk for misuse.
    \item[] Guidelines:
    \begin{itemize}
        \item The answer \answerNA{} means that the paper poses no such risks.
        \item Released models that have a high risk for misuse or dual-use should be released with necessary safeguards to allow for controlled use of the model, for example by requiring that users adhere to usage guidelines or restrictions to access the model or implementing safety filters. 
        \item Datasets that have been scraped from the Internet could pose safety risks. The authors should describe how they avoided releasing unsafe images.
        \item We recognize that providing effective safeguards is challenging, and many papers do not require this, but we encourage authors to take this into account and make a best faith effort.
    \end{itemize}

\item {\bf Licenses for existing assets}
    \item[] Question: Are the creators or original owners of assets (e.g., code, data, models), used in the paper, properly credited and are the license and terms of use explicitly mentioned and properly respected?
    \item[] Answer: \answerYes{}
    \item[] Justification: All datasets and baseline algorithms used are publicly available and cited at the original sources (\Cref{sec:appen-impl,sec:exp-setting}); we use them under their respective licenses and terms of use.
    \item[] Guidelines:
    \begin{itemize}
        \item The answer \answerNA{} means that the paper does not use existing assets.
        \item The authors should cite the original paper that produced the code package or dataset.
        \item The authors should state which version of the asset is used and, if possible, include a URL.
        \item The name of the license (e.g., CC-BY 4.0) should be included for each asset.
        \item For scraped data from a particular source (e.g., website), the copyright and terms of service of that source should be provided.
        \item If assets are released, the license, copyright information, and terms of use in the package should be provided. For popular datasets, \url{paperswithcode.com/datasets} has curated licenses for some datasets. Their licensing guide can help determine the license of a dataset.
        \item For existing datasets that are re-packaged, both the original license and the license of the derived asset (if it has changed) should be provided.
        \item If this information is not available online, the authors are encouraged to reach out to the asset's creators.
    \end{itemize}

\item {\bf New assets}
    \item[] Question: Are new assets introduced in the paper well documented and is the documentation provided alongside the assets?
    \item[] Answer: \answerNA{}
    \item[] Justification: The paper does not release any new datasets or pre-trained models; we will release anonymized code with the camera-ready version, with documentation following standard open-source practice.
    \item[] Guidelines:
    \begin{itemize}
        \item The answer \answerNA{} means that the paper does not release new assets.
        \item Researchers should communicate the details of the dataset\slash code\slash model as part of their submissions via structured templates. This includes details about training, license, limitations, etc. 
        \item The paper should discuss whether and how consent was obtained from people whose asset is used.
        \item At submission time, remember to anonymize your assets (if applicable). You can either create an anonymized URL or include an anonymized zip file.
    \end{itemize}

\item {\bf Crowdsourcing and research with human subjects}
    \item[] Question: For crowdsourcing experiments and research with human subjects, does the paper include the full text of instructions given to participants and screenshots, if applicable, as well as details about compensation (if any)?
    \item[] Answer: \answerNA{}
    \item[] Justification: The paper does not involve crowdsourcing or research with human subjects.
    \item[] Guidelines:
    \begin{itemize}
        \item The answer \answerNA{} means that the paper does not involve crowdsourcing nor research with human subjects.
        \item Including this information in the supplemental material is fine, but if the main contribution of the paper involves human subjects, then as much detail as possible should be included in the main paper. 
        \item According to the NeurIPS Code of Ethics, workers involved in data collection, curation, or other labor should be paid at least the minimum wage in the country of the data collector. 
    \end{itemize}

\item {\bf Institutional review board (IRB) approvals or equivalent for research with human subjects}
    \item[] Question: Does the paper describe potential risks incurred by study participants, whether such risks were disclosed to the subjects, and whether Institutional Review Board (IRB) approvals (or an equivalent approval/review based on the requirements of your country or institution) were obtained?
    \item[] Answer: \answerNA{}
    \item[] Justification: The paper does not involve research with human subjects.
    \item[] Guidelines:
    \begin{itemize}
        \item The answer \answerNA{} means that the paper does not involve crowdsourcing nor research with human subjects.
        \item Depending on the country in which research is conducted, IRB approval (or equivalent) may be required for any human subjects research. If you obtained IRB approval, you should clearly state this in the paper. 
        \item We recognize that the procedures for this may vary significantly between institutions and locations, and we expect authors to adhere to the NeurIPS Code of Ethics and the guidelines for their institution. 
        \item For initial submissions, do not include any information that would break anonymity (if applicable), such as the institution conducting the review.
    \end{itemize}

\item {\bf Declaration of LLM usage}
    \item[] Question: Does the paper describe the usage of LLMs if it is an important, original, or non-standard component of the core methods in this research? Note that if the LLM is used only for writing, editing, or formatting purposes and does \emph{not} impact the core methodology, scientific rigor, or originality of the research, declaration is not required.
    %this research?
    \item[] Answer: \answerNA{}
    \item[] Justification: LLMs were not used as part of the core methodology of this research; any use was limited to writing assistance and does not affect the scientific content.
    \item[] Guidelines:
    \begin{itemize}
        \item The answer \answerNA{} means that the core method development in this research does not involve LLMs as any important, original, or non-standard components.
        \item Please refer to our LLM policy in the NeurIPS handbook for what should or should not be described.
    \end{itemize}

\end{enumerate}

\end{document}